\documentclass{article}

\usepackage{microtype}
\usepackage{graphicx}
\usepackage{subcaption}
\usepackage{booktabs} 
\usepackage[textsize=tiny]{todonotes}
\usepackage{enumitem}
\usepackage{hyperref}

\usepackage[accepted]{icml2026}

\usepackage{amsmath}
\usepackage{amssymb}
\usepackage{mathtools}
\usepackage{amsthm}

\usepackage[capitalize,noabbrev]{cleveref}

\usepackage{xcolor}
\theoremstyle{plain}
\newtheorem{theorem}{Theorem}[section]

\newtheorem{lemma}[theorem]{Lemma}

\theoremstyle{definition}
\newtheorem{definition}[theorem]{Definition}
\newtheorem{assumption}[theorem]{Assumption}
\theoremstyle{remark}
\newtheorem{remark}[theorem]{Remark}

\usepackage[textsize=tiny]{todonotes}

\icmltitlerunning{The Stability of Singular Distribution}

\begin{document}

\twocolumn[

  \icmltitle{The Stability of Singular Distribution: A Spectral Perspective \\ on the Two-Phase Dynamics of Language Model Pre-training}

  \icmlsetsymbol{equal}{*}

  \begin{icmlauthorlist}
    \icmlauthor{Hongtao Zhang}{ict,sais,equal}
    \icmlauthor{Wenjie Zhou}{ict,ucas,equal}
    \icmlauthor{Chenxi Jia}{ict,seu}
    \icmlauthor{Wei Chen}{ict,ucas}
    \icmlauthor{Xueqi Cheng}{ict,ucas}
  \end{icmlauthorlist}
  
  \icmlaffiliation{sais}{School of Advanced Interdisciplinary Sciences, University of Chinese Academy of Sciences, Beijing, China}
  \icmlaffiliation{ict}{State Key Laboratory of AI Safety, Institute of Computing Technology, Chinese Academy of Sciences, Beijing, China}
  \icmlaffiliation{ucas}{University of Chinese Academy of Sciences, Beijing, China}
  \icmlaffiliation{seu}{School of Mathematics, Southeast University, Nanjing, China}

  \icmlcorrespondingauthor{Wei Chen}{chenwei2022@ict.ac.cn}

  \icmlkeywords{Machine Learning, ICML}

  \vskip 0.3in
]



\printAffiliationsAndNotice{\icmlEqualContribution}  

\begin{abstract}

Large language model pre-training typically exhibits a two-phase trajectory: a fast initial loss drop followed by a prolonged slow improvement. We identify an underlying spectral phenomenon, Stability of Singular Distribution (SoSD), where the trace-normalized singular value spectrum stabilizes early, even as parameter matrices continue to evolve. We demonstrate that synchronization between SoSD and the slow-descent regime is widely observed across diverse architectures (GPT-2, LLaMA) and settings, including various schedules (Step-wise, WSD, Cosine Decay), weight decays, and optimizers (AdamW, Muon). By analyzing a simplified Transformer, we prove that growing weight norms inevitably precipitate an early SoSD threshold, after which the rate of loss decrease becomes theoretically bounded by the variation in the singular distribution. We further interpret strategies like WSD and Muon through their ability to modulate the SoSD scale, offering a spectral lens for understanding efficient pre-training dynamics.

\end{abstract}

\section{Introduction}
Large Language Models (LLMs) have established themselves as the cornerstone of modern artificial intelligence \citep{brown2020language, achiam2023gpt}, achieving unprecedented scalability and generalization by leveraging the Transformer architecture \citep{vaswani2017attention} as their foundational backbone. However, the optimization dynamics governing their pre-training remain enigmatic, particularly regarding the temporal evolution of the training process.

A ubiquitous observation recorded across major technical reports \citep{touvron2023llama, chowdhery2023palm, zhang2022opt} is the characteristic two-phase trajectory of the training loss: an initial regime of precipitous decay followed by a prolonged period of asymptotic, heavy-tailed improvement. 
While this ``fast-then-slow'' two-phase behavior is empirically taken for granted as the standard convergence pattern of Large Language Models (LLMs), the underlying theoretical mechanisms driving this transition and specifically, what mechanistic factors dictate the onset of the slow-descent phase remain fundamentally under-explored.
\begin{figure}[t]
  \vskip 0.2in
  \begin{center}
    \centerline{\includegraphics[width=1.05\linewidth]{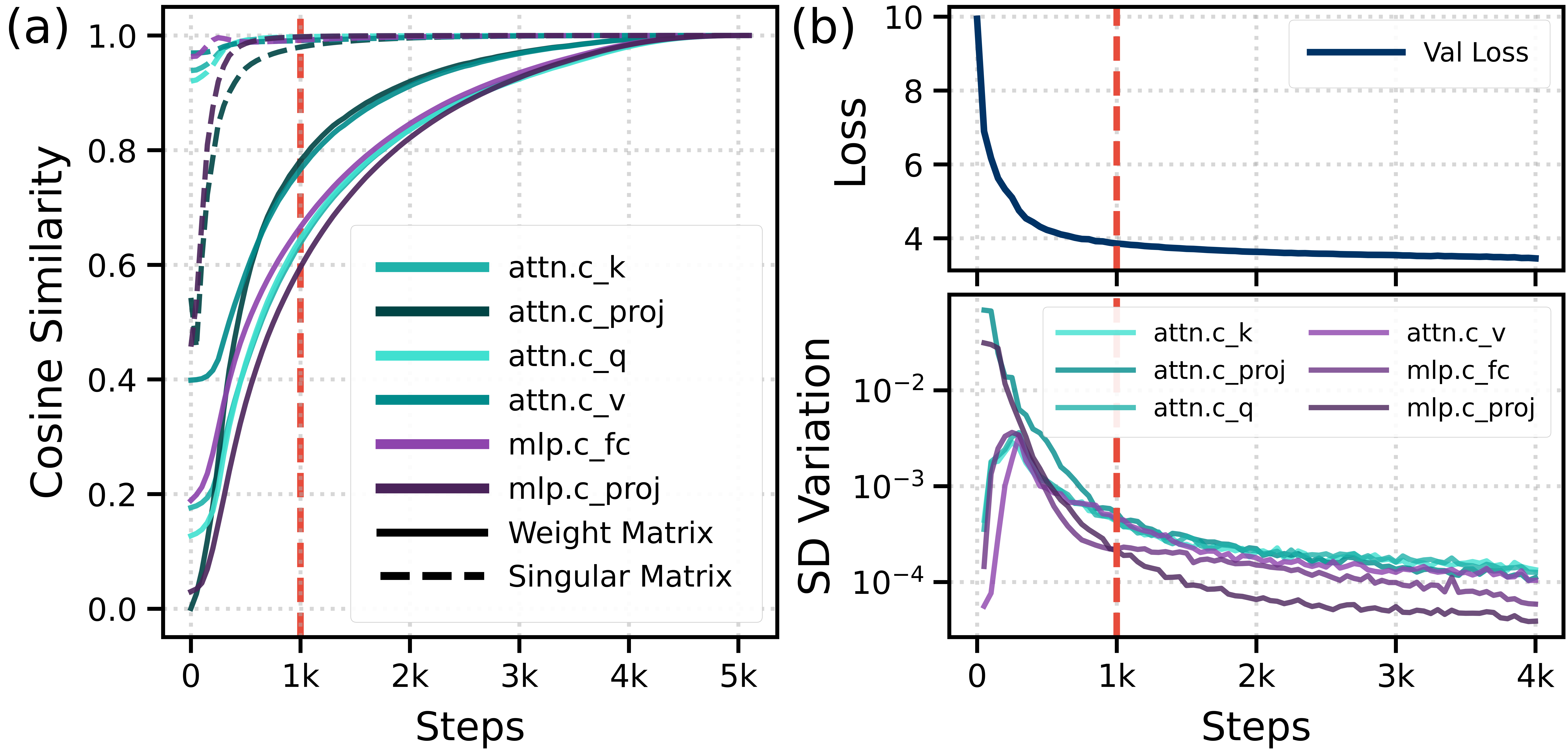}}
    \caption{
        \textbf{Identification of the Stability of Singular Distribution (SoSD) phenomenon on GPT-2 Small.}
        \textbf{(a)} Evolution of cosine similarity between current states and final states. The singular distributions (dashed lines) stabilize significantly earlier than the parameter matrices themselves (solid lines).
        \textbf{(b)} Synchronization between Validation Loss (top) and Singular Distribution Variation (SD Variation, bottom). Vertical red dashed lines mark the approximate steps where singular value matrices stabilize, coinciding with the transition into the slow descent regime.
    }
    \label{fig_Intro}
  \end{center}
\vskip -0.4in
\end{figure}

Existing analyses often rely on restricted task formulations, such as in-context learning \citep{olsson2022context,bietti2023birth,zhang2025training} or linear regression \citep{zhang2024trained}, to derive tractable convergence bounds. 
Recently, the focus has shifted towards more granular, multi-stage interpretations of Transformer optimization \citep{zhou2022towards,yao2025analysis,zhang2025complexity}. 
Notably, recent work on Condensation to Rank Collapse \citep{chen2025condensation} employs a gradient flow framework on linearized attention, identifying a two-stage transition from parameter condensation to asymptotic rank collapse under small initialization. 
Motivated by these insights, we propose to shift the lens towards the \textit{temporal dynamics} of the spectral evolution. Crucially, this perspective allows us to investigate the mechanistic synchronization between stability of singular distribution and the macroscopic saturation of the loss function, addressing a fundamental question: 

\begin{center}
\textit{What intrinsic mechanism governs the transition from fast learning to slow saturation, and how does the spectral evolution of the parameters dictate this regime shift?}
\end{center}

To investigate the mechanism underlying this transition, we analyze the spectral dynamics of the parameter matrices throughout the pre-training process. 
We identify a phenomenon termed the \textbf{Stability of Singular Distribution (SoSD)}, where the normalized singular value spectrum stabilizes significantly earlier than the parameter matrices themselves
(Figure \ref{fig_Intro}(a)), and the onset of this stability aligns with the validation loss entering a plateau, exhibiting a tight synchronization with optimization saturation (Figure \ref{fig_Intro}(b)). 
Our analysis reveals that this spectral stabilization closely characterizes the shift from the fast to the slow descent phase. 
Our specific contributions are as follows:

\begin{enumerate}
    \item \textbf{Identification of SoSD}: 
    We identify the Stability of Singular Distribution (SoSD) phenomenon, observing that the singular value distribution enters a stable state significantly prior to the stability of parameter matrices(Figure \ref{fig_Intro}(a)). Across GPT-2 and LLaMA families, we report a synchronization that the onset of SoSD aligns with the loss function's transition to the slow descent regime(Figure \ref{fig_Intro}(b)).
    
    \item \textbf{Theoretical Analysis of SoSD}: We establish a theoretical framework to clarify the mechanism of SoSD. We first establish the emergence of SoSD (Theorem~\ref{thm:ssd_stability}) contingent upon the Non-Degeneracy of parameters and Gradient Boundedness (Assumption~\ref{ass:regularity}). Building on this, we incorporate Smoothness and Margin Conditions (Assumption~\ref{ass:smoothness_margin}) to demonstrate that loss descent is dynamically coupled with singular distribution variation: while large variation is associated with rapid loss decay, the onset of SoSD strictly bounds the subsequent loss reduction (Theorem~\ref{thm:two_phase}).

    \item \textbf{Interpreting Pre-training Strategies via SoSD}:
    We interpret pre-training strategies through the lens of SoSD dynamics and the derived stability bound $\varepsilon \propto \eta/\|W\|$, where $\eta$ and $\|W\|$ denote the learning rate and the norm of parameters, respectively. We demonstrate that learning rate schedules (step-wise decay and continuous annealing) facilitate optimization by tightening this bound, thereby mitigating the SoSD-related constraint to enable further loss minimization. 
    Conversely, we find that Weight Decay facilitates loss descent by suppressing the growth of the weight norm; this mechanism relaxes the stability constraint, thereby permitting larger updates to the singular distribution. Finally, we validate SoSD with the Muon optimizer, observing that the SoSD phenomenon persists alongside its superior training efficiency.
\end{enumerate}

\paragraph{Conflict of Interest Disclosure.}
The authors declare no financial conflicts of interest related to this work.

\begin{figure*}[t!]
  \vskip 0.2in
  \begin{center}
    \centerline{\includegraphics[width=2\columnwidth]{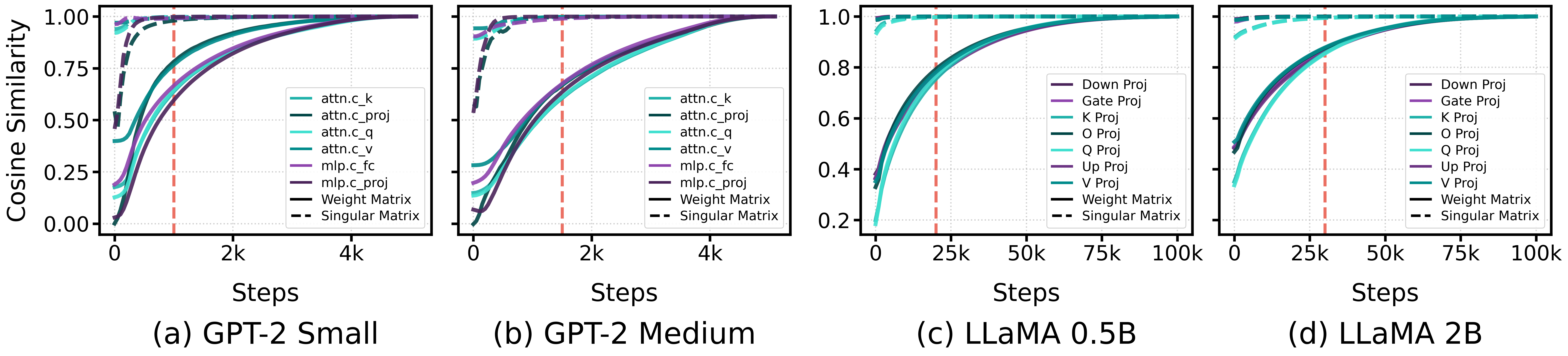}}
    \caption{
        \textbf{Evolution of cosine similarity between current parameters($t$) and their final states($T$) during pre-training.} The plots analyze cosine similarity $\mathrm{cos}\langle W_t, W_T \rangle$ and $\mathrm{cos}\langle \Sigma_t, \Sigma_T \rangle$ across four models: (a) GPT-2 Small, (b) GPT-2 Medium, (c) LLaMA 0.5B, and (d) LLaMA 2B. Solid lines represent the weight matrices ($W$), while dashed lines represent the singular value matrices ($\Sigma$). Colors denote different projection layers. The vertical red dashed lines mark the approximate steps where the singular value matrices stabilize.
    }
    \label{fig_cossim}
  \end{center}
\end{figure*}

\vspace{-0.4\baselineskip}
\section{Related Work}

\paragraph{Training Dynamics of Transformers}
Understanding the optimization trajectory of Transformers remains a formidable challenge given the interplay between non-convex objectives and massive parameter scales. Early theoretical efforts primarily dissected the dynamics of simplified, single-layer attention mechanisms. For instance, \citet{tian2023scan} and \citet{snell2021approximating} characterized how gradient descent drives attention heads to capture co-occurrence patterns or mimic Seq2Seq algorithms, while \citet{li2023transformerslearntopicstructure} demonstrated the learnability of topic models within a BERT-like framework \citep{devlin2019bert}. 
Building on these foundational settings, recent studies have extended dynamic analysis to more complex distributions. For instance, the staged learning phases in two-mixture linear classification have been detailed \citep{yang2025transformers}. Similarly, in the realm of logical reasoning, theoretical proofs have established how attention and linear layers evolve to solve regular language tasks via Chain-of-Thought \citep{huang2025transformers}.
More recently, the focus has shifted towards In-Context Learning (ICL) as a testbed for understanding dynamics. A substantial body of work has adopted the linear regression setting to theoretically analyze how Transformers implement gradient descent-like algorithms during inference \citep{akyurek2022learning, von2023transformers, mahankali2023one, zhang2024trained}. 
Refining this perspective, recent research demonstrates that Transformers can transcend simple algorithms to learn latent representations, implicitly performing ridge regression to generalize to unseen tasks \citep{yang2024context}.
Parallel to this, mechanistic interpretability research has traced the emergence of specific structural components during training, such as induction heads \citep{olsson2022context, reddy2023mechanistic} and memory retrieval circuits \citep{bietti2023birth, cabannes2024learning}.
Complementing these granular analyses of specific capabilities and circuits, our work seeks to characterize the \textit{macroscopic temporal evolution} of the loss landscape in general pre-training scenarios. We aim to bridge the mechanistic link between the spectral evolution of parameters and the global phenomenon of loss saturation (the two-phase transition), offering a unified spectral perspective that governs the training efficiency of standard language models.

\paragraph{Implicit Regularization and Structural Dynamics}
Theoretical investigations into Transformer optimization have extensively characterized how gradient-based learning induces specific structural properties in weight matrices.
A central theme is the emergence of low-complexity solutions, famously observed as "Rank Collapse" \citep{dong2021attention}, where the effective rank of parameters diminishes over time. This phenomenon is widely interpreted as a form of implicit regularization, where the optimizer naturally favors low-rank or max-margin solutions even without explicit constraints \citep{gunasekar2017implicit, soudry2018implicit, arora2019implicit,neyshabur2017implicit}.
Supporting this view, recent non-asymptotic analysis for next-token prediction establishes that both feed-forward and attention layers converge to such max-margin solutions with linear rates \citep{huang2024non}.
Complementary to this structural view, the dynamic interplay between optimizer sharpness and stability has been analyzed through the "Edge of Stability" framework \citep{cohen2021gradient, ahn2022understanding,damian2022self}.
However, a tension exists between the tendency towards low-rank structures and the empirical observation of monotonic norm growth during pre-training \citep{merrill2021effects}.
Gradient flow theories attempt to reconcile these aspects by modeling training as a process of incremental rank accumulation or balanced flow \citep{saxe2019mathematical, gidel2019implicit}.
Most relevantly, \citet{chen2025condensation} recently employed a gradient flow framework on \textit{linearized} Transformers to identify a transition from parameter condensation to asymptotic rank collapse.
Building on these foundational insights, we extend the analysis to the spectral dynamics of standard attention. We identify the Stability of Singular Distribution (SoSD) not as a final low-rank state, but as an \textit{early-onset} kinetic bottleneck. This perspective provides a mechanistic ground for the "fast-then-slow" two-phase saturation observed in practical pre-training.

\section{The Stability Phenomenon in Singular Distribution}
\label{section_sosd}
\subsection{Experimental Setting}
\textbf{Models and Datasets.} 
We conduct pre-training experiments on two widely adopted decoder-only model families:  
\begin{itemize}[leftmargin=*, topsep=2pt, itemsep=2pt]
    \item \textbf{GPT-2 on FineWeb}: We train GPT-2 Small (124M) and Medium (355M) models on dataset using the highly optimized \texttt{nano-gpt} benchmark training recipe \footnote{\url{https://github.com/KellerJordan/modded-nanogpt}}, establishing a rigorous standard for small-to-medium scale language modeling \citep{radford2019language}.
    \item \textbf{LLaMA on C4}: We train 0.5B and 2B parameter LLaMA models on the Colossal Clean Crawled Corpus (C4) to evaluate architectural scalability \citep{touvron2023llama}.
\end{itemize}
\vspace{-0.4\baselineskip}

\textbf{Optimization Settings. }All models are trained using the AdamW optimizer with hyperparameters $\beta_1=0.9$ and $\beta_2=0.95$, with additional comparative runs using the Muon optimizer specifically for GPT-2 Small. We employ diverse learning rate schedules (Step Decay and Warmup-Stable-Decay for GPT-2, Cosine decay for LLaMA) and ablate weight decay on the LLaMA 0.5B model \citep{loshchilov2017decoupled}. Full details are provided in Appendix \ref{app:exp_details}.


\subsection{The Emergence of Stability of Singular Distribution}


We investigate the convergence trajectories of the parameter matrices $W$ and their associated singular value spectra $\Sigma$ by monitoring their cosine similarity to the final trained states ($T$) across GPT-2 and LLaMA architectures (Figure~\ref{fig_cossim}).

\textbf{Spectral Decoupling.} A distinct divergence in convergence rates is observed across all models. While the parameter alignment $\cos\langle W_t, W_T \rangle$ evolves gradually, indicating that weight entries remain distant from their final configuration, the spectral alignment $\cos\langle \Sigma_t, \Sigma_T \rangle$ rapidly saturates near $1$. This implies that the \textit{relative shape} of the singular value spectrum stabilizes significantly earlier than the parameters themselves. To formalize this decoupling, we introduce the following definition:

\begin{definition}[Singular Distribution, SD]
\label{def:singular_distribution}
Let $\Sigma_t$ denote the diagonal matrix of singular values derived from $W_t$. The \textit{Singular Distribution} is defined as the trace-normalized spectrum: $\hat{\Sigma}_t \triangleq {\Sigma_t}/{\operatorname{tr}(\Sigma_t)}.$

\end{definition}

Since cosine similarity is scale-invariant, the saturation of $\cos\langle \Sigma_t, \Sigma_T \rangle$ is equivalent to the stabilization of $\hat{\Sigma}_t$. We term this phenomenon where the singular distribution achieves stationarity prior to the parameter matrices as the \textbf{Stability of Singular Distribution (SoSD)}.

\subsection{Synchronized Dynamics of SoSD and Training Loss}

\begin{figure*}[t!]
  \vskip 0.2in
  \begin{center}
    \centerline{\includegraphics[width=2\columnwidth]{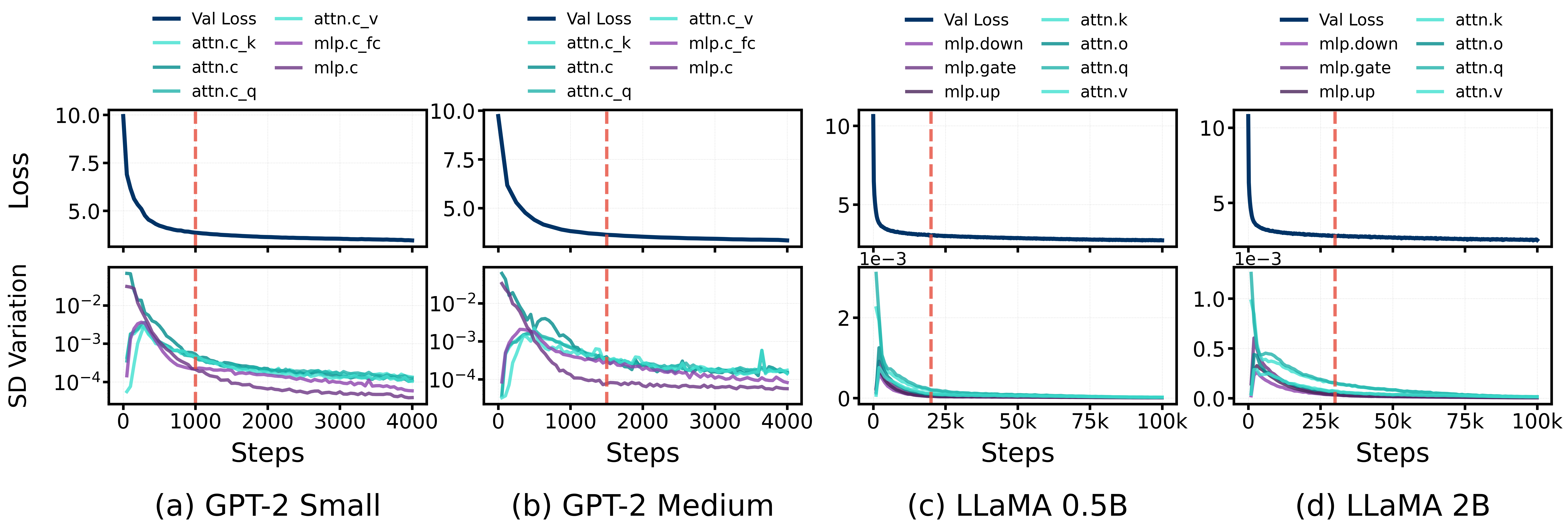}}
    \caption{
        \textbf{Synchronization between Validation Loss and Singular Distribution Variation.} 
        The figures illustrate the training dynamics across four models: 
        \textbf{(a)}~GPT-2 Small, \textbf{(b)}~GPT-2 Medium, \textbf{(c)}~LlaMA 0.5B, and \textbf{(d)}~LlaMA 2B. 
        The \textbf{top row} displays the Validation Loss, while the \textbf{bottom row} shows the Singular Distribution Variation for various projection layers. 
         Vertical red dashed lines mark the approximate steps where singular value matrices stabilize.
    }
    \label{fig_loss_sosd}
  \end{center}
\end{figure*}

To quantify the temporal evolution of the singular spectrum, we define the \textbf{Singular Distribution Variation (SD Variation)} as the magnitude of the update to the normalized spectrum:
$\Delta \Sigma_t := \left\| \hat{\Sigma}_{t+1} - \hat{\Sigma}_{t} \right\|_F$,
where $\hat{\Sigma}_t$ is the singular distribution defined in Definition~\ref{def:singular_distribution}. We investigate the correlation between this spectral variation and the validation loss trajectory in Figure~\ref{fig_loss_sosd}.


Our analysis reveals a tight synchronization between the stability of the singular distribution and the regime of loss descent. Specifically, the training dynamics (exemplified in Figure~\ref{fig_loss_sosd}(a)) decouple into two distinct phases governed by the magnitude of $\Delta \Sigma_t$:

\begin{itemize}[leftmargin=*, topsep=2pt, itemsep=2pt]
    \item \textbf{Phase I: Singular Distribution Restructuring (Fast Descent).}
    In the initial stage ($t < 1000$), $\Delta \Sigma_t$ exhibits a transient impulse, peaking at magnitudes of approximately $10^{-2}$. This period of active spectral restructuring coincides with the rapid decay of the validation loss (from $\sim 10$ to $\sim 4.0$). This suggests that the steepest descent in the loss landscape is mechanistically driven by significant shifts in the relative singular value distribution.

    \item \textbf{Phase II: Singular Distribution Metastability (Slow Descent).}
    For $t > 1000$, $\Delta \Sigma_t$ decays rapidly and settles into a metastable floor ($\approx 10^{-4}$), marking the onset of SoSD. Crucially, this spectral stabilization precisely aligns with the transition of the loss function into the asymptotic, heavy-tailed regime. We note that MLP layers consistently exhibit lower Singular Distribution Variation than Attention layers, yet both components synchronize their entry into the SoSD regime with the saturation of the loss.
\end{itemize}

As we illustrated in Figure~\ref{fig_loss_sosd}, this synchronization is universal across architectures (GPT-2, LLaMA) and modules, indicating that the SoSD is not merely an artifact of specific layers but a global kinetic constraint characterizing the saturation of Transformer pre-training.

\section{Theoretical Analysis}
\label{section_thm}
In this section, we will theoretically analyze the inevitability of SoSD within a single-layer, single-head Transformer setting and investigate the rationale for the synchronized dynamics of SoSD and the Transformer training loss.

\subsection{Setup}

\paragraph{Task and Loss Function}
We consider a sequence classification setting. Let $X \in \mathbb{R}^{n \times d}$ denote the input sequence, where $n$ is the sequence length and $d$ is the feature dimension \citep{tarzanagh2023transformers,tian2023scan}. The corresponding labels are represented by a matrix $Y \in \mathbb{R}^{n \times C}$, where each row $Y_i$ is a one-hot vector indicating the class label at position $i$. Formally, the elements of $Y$ are defined as:
\begin{equation}
    \label{yi_formula}
    Y_{i,c} = \mathbb{I}[c = y_i] = 
    \begin{cases} 
    1, & \text{if } c = y_i, \\
    0, & \text{if } c \neq y_i,
    \end{cases}
\end{equation}
where $y_i \in \{1, \ldots, C\}$ denotes the ground-truth class index for the $i$-th position.

We adopt the cross-entropy loss for training. The empirical risk over the sequence is defined as:
$
    \mathcal{L} = \frac{1}{n} \sum_{i=1}^{n} \ell_i = -\frac{1}{n} \sum_{i=1}^{n} \sum_{c=1}^{C} Y_{i,c} \log P_{i,c},$
where $Y_{i,c}$ denotes the one-hot encoded ground-truth label. 
Since $Y_{i,c} = \mathbb{I}[c = y_i]$ according to (\ref{yi_formula}), the loss can be equivalently simplified to:
\begin{equation}
    \mathcal{L} = -\frac{1}{n} \sum_{i=1}^{n} \log P_{i, y_i}.
\end{equation}

\paragraph{Model Architecture}    \label{model arch}
We employ a 1-layer single-head transformer architecture \citep{chen2025condensation,yang2024context,jiao2025transformers}. Given the input sequence $X \in \mathbb{R}^{n \times d}$, we denote the query, key, and value matrices via linear transformations:
$Q = X W_Q, \quad K = X W_K, \quad V = X W_V,$
where $W_Q, W_K, W_V \in \mathbb{R}^{d \times d}$ are learnable weight matrices.

The attention mechanism is formalized by the score matrix $M$ and the attention matrix $A$:
$$M= \frac{Q K^\top}{\sqrt{d}}, \quad A= \mathrm{softmax}(M) .$$
The resulting hidden representations are then computed as:
$
    H = A V \in \mathbb{R}^{n \times d}.
$

To map these representations to the output space, we introduce a fixed (non-trainable) projection matrix $W_C \in \mathbb{R}^{d \times C}$. The logits $Z$ and output probabilities $P$ are given by:
$$
    Z = H W_C, \quad P = \mathrm{softmax}(Z),
$$
where $Z, P \in \mathbb{R}^{n \times C}$.

\paragraph{Optimization}
We optimize the model parameters using full-batch Gradient Descent (GD) with constant learning rate $\eta>0$. Denote
$\theta_t = \{ W_Q(t), W_K(t), W_V(t) \}.$ 
The parameters are updated according to the gradient descent rule: $$\theta_{t+1} = \theta_t - \eta \nabla_{\theta} \mathcal{L}(\theta_t).$$

\paragraph{Initialization}
We adopt small initialization strategy \citep{yao2025analysis,chen2025condensation,giorlandino2025two}. Specifically, the entries of the parameter matrices at $t=0$ are sampled independently from a Gaussian distribution: $[W]_{i,j} \sim \mathcal{N}(0, \sigma^2), \quad W \in \{W_Q, W_K, W_V\},$
where $\sigma \ll 1$ denotes the standard deviation.

\subsection{Stability of Singular Distribution}

To theoretically characterize the Stability of Singular Distribution (SoSD), we first introduce the necessary technical assumptions regarding the evolution of weight norms and the regularity of the gradients.

\begin{assumption}[Strictly Increasing Norms]
\label{ass:increasing_norm}
The norm of the weight matrices strictly increases over time, i.e.  $\forall W \in \{W_Q, W_K, W_V\}$ and $t$, $\|W(t+1)\|> \|W(t)\|$ is strictly monotonically increasing.
\end{assumption}

\begin{remark}
    Due to the absence of penalties induced by weight decay, the growth of matrix norms remains unconstrained \citep{andriushchenko2023we}. Consequently, weight norms tend to increase monotonically during training, a phenomenon widely observed in deep learning optimization.
\end{remark}

\begin{assumption}[Non-Degeneracy and Gradient Boundedness]
\label{ass:regularity}
We impose the following regularity conditions on the model parameters and gradients:
\begin{enumerate}[leftmargin=*, topsep=2pt, itemsep=2pt]
    \item Bounded Condition Number: \label{item condition number}For all $W \in \{W_Q, W_K, W_V\}$, the matrices are full-rank and their condition numbers are bounded, i.e. there exists a constant $\kappa_\bullet > 0$ such that:
    \begin{equation*}
        \kappa(W_\bullet) := \frac{\sigma_{\max}(W_\bullet)}{\sigma_{\min}(W_\bullet)} \le \kappa_\bullet, \bullet \in \{Q,K,V\}.
    \end{equation*}

    \item Non-vanishing Gradients: \label{item nonvanishing grad}The backpropagated gradient with respect to the hidden representations $H$ is non-vanishing, i.e.
    $\|G_H\| = \left\| \frac{\partial \mathcal{L}}{\partial H} \right\| > 0.$
    
    \item Bounded Gradient Norms: \label{item bounded gradnorm}The gradients with respect to the weight matrices are bounded. There exists a constant $G > 0$ such that for all $W \in \{W_Q, W_K, W_V\}$:
    $\left\| \frac{\partial \mathcal{L}}{\partial W} \right\| \le G.$
\end{enumerate}
\end{assumption}

\begin{remark}
These assumptions are well-grounded in the theoretical analysis of training dynamics. 
Regarding Assumption \ref{ass:regularity}.\ref{item condition number}, parameter matrices are typically full-rank (or nearly full-rank) during the pre-training phase, which ensures that their condition numbers remain bounded. 
Assumption \ref{ass:regularity}.\ref{item nonvanishing grad} holds for any model actively in the learning phase. 
Finally, Assumption \ref{ass:regularity}.\ref{item bounded gradnorm} is theoretically guaranteed within any finite training horizon: since parameters reside in a compact set and the objective function is continuous, gradient norms strictly remain finite despite potential fluctuations.
\end{remark}

\begin{theorem}[Stability of Singular Distribution]
\label{thm:ssd_stability}
Consider the model defined in Section~\ref{model arch}. Under Assumption~\ref{ass:increasing_norm} and Assumption~\ref{ass:regularity}, for stability bound $\varepsilon(W)=O(\frac{\eta}{\|W\|}), W \in \{W_Q,W_K,W_V\}$ , the following hold:

\begin{enumerate}[leftmargin=*, topsep=2pt, itemsep=2pt]
    \item For \textbf{the Value Matrix ($W_V$)}: 
    There exists a threshold time $T_V$ satisfying
    \begin{equation*}
        T_V = O\left(\frac{(1+\sqrt{d})\,\eta G - \varepsilon v(0)}{\varepsilon\, C_V\sqrt{d}}\right),
    \end{equation*}
    such that for all $t \ge T_V$, the singular distribution stabilizes:
    \begin{equation*}
        \left\| \hat{\Sigma}_{t+1}(W_V) - \hat{\Sigma}_{t}(W_V) \right\|_F < \varepsilon(W_V).
    \end{equation*}

    \item For \textbf{Query and Key Matrices ($W \in \{W_Q, W_K\}$)}: 
    There exists a threshold time $T_{QK}$ which scales as:
    \begin{equation*}
    \begin{split}
        T_{QK} = O\Bigg( \bigg( C_0^2 + 2 C_0 \Lambda(\varepsilon) \eta G \bigg)^{1/2} \Bigg),
    \end{split}
    \end{equation*}
    such that for all $t \ge T_{QK}$, the singular distribution stabilizes:
    \begin{equation*}
        \left\| \hat{\Sigma}_{t+1}(W) - \hat{\Sigma}_{t}(W) \right\|_F < \varepsilon(W),
    \end{equation*}
    where $\Lambda(\varepsilon)=\ln(1+\sqrt{d})-\ln(\varepsilon\, q(0)\sqrt{d})$.
\end{enumerate}

Consequently, there exists a sufficiently large constant $C > 0$ such that the global threshold time $T^*$ is defined as:
\begin{equation*}
    T^* = C\max \left\{ 
        \frac{(1+\sqrt{d})\,\eta G}{\varepsilon\, C_V\sqrt{d}}, 
        \sqrt{C_0^2 + 2 C_0 \Lambda(\varepsilon) \eta G} 
    \right\}.
\end{equation*}
For all $t \ge T^*$, the SoSD holds for all matrices $W \in \{W_Q, W_K, W_V\}$.

\end{theorem}

The proof is provided in Appendix \ref{proof_for_thm1}. Theorem \ref{thm:ssd_stability} establishes the mechanistic origin of SoSD. Under small initialization, the relatively small parameter norms allow for rapid SD Variation. However, as training progresses, the monotonic growth of weight norms (Assumption \ref{ass:increasing_norm}) progressively suppresses the relative magnitude of gradient updates. This scaling effect naturally enforces a tighter bound on SD variation ($\varepsilon \propto \eta / \|W\|$), effectively compelling the system into the SoSD regime.



\subsection{2-Phase Analysis of Training Loss}

\begin{definition}[Gap Function]
\label{def:gap_function}
Let $u \in \mathbb{R}^m$. Define the index of the maximum element as $j^* := \arg\max_{j} u_j$. The gap function $\mathrm{gap}(\cdot)$ is defined as the difference between the largest and the second-largest entries:
$\mathrm{gap}(u) := u_{j^*} - \max_{j \neq j^*} u_j \ge 0.$
\end{definition}

\begin{assumption}[Smoothness and Margin Conditions]
\label{ass:smoothness_margin}
We posit the following regularity conditions regarding the loss landscape and asymptotic margins:
\begin{enumerate}[leftmargin=*, topsep=2pt, itemsep=2pt]
    \item Smoothness: \label{betasmooth}There exists $T_{\beta}$ such that for all $t > T_{\beta}$, the loss function is locally $\beta$-smooth, i.e. for all $t > T_{\beta}$, for all $W \in \{W_Q, W_K, W_V\}$:
    $\|\nabla_W \mathcal{L}(W(t)) - \nabla_W \mathcal{L}(W(t+1))\| \le \beta \|W(t) - W(t+1)\|.$


    \item Attention Margin. \label{attnmargin}
    For each $t > \max\{T^*, T_\beta\}$, denote
    $
    \bar{M} := \frac{M}{\|W_Q\|_* \|W_K\|_*}.
    $
    We assume that the attention mechanism separates the tokens with a positive margin: 
    $
    \gamma_{\min}
    :=
    \min_i \operatorname{gap}\big(\bar{M}_{i,:}\big)
    > 0.
    $
    
    \item Logit Margin. \label{logitmargin}
    For each $t > \max\{T^*, T_\beta\}$, denote $\bar{Z} := \frac{Z}{\|W_V\|_*}$. For the $i$-th token, define its normalized logit margin as $\omega_i := \bar{Z}_{i,y_i} - \max_{c \neq y_i} \bar{Z}_{i,c}$. We assume that the minimum logit margin is strictly positive: $\omega_{\min} := \min_i \omega_i > 0$.
\end{enumerate}
\end{assumption}



\begin{remark}
Assumption \ref{ass:smoothness_margin}.\ref{betasmooth} is a standard local smoothness condition used to control the loss landscape along the training trajectory once it has moved beyond the initial transient; the threshold $T_{\beta}$ is an analytical threshold for invoking this local regularity, not a tunable quantity or an exact prediction of the empirical phase boundary. 
 Assumption \ref{ass:smoothness_margin}.\ref{attnmargin} and Assumption \ref{ass:smoothness_margin}.\ref{logitmargin} are phase-specific margin conditions required only for the Phase II slow-descent analysis, i.e., for $t>\max\{T^*,T_{\beta}\}$. 
They are not assumed to hold from initialization or throughout the entire training trajectory. These margin conditions are motivated by the saturation property of the softmax function and the asymptotic behavior of cross-entropy minimization, which promote token-level concentration and class separation in the late training regime.
\end{remark}

\begin{theorem}[Two-Phase Dynamics]
\label{thm:two_phase}
Denote the loss decrease as $\Delta \mathcal{L}(t) := \mathcal{L}_t - \mathcal{L}_{t+1}$. Consider the model architecture defined in Section~\ref{model arch}. Under Assumption~\ref{ass:increasing_norm} and Assumption~\ref{ass:smoothness_margin}, the training dynamics exhibit two distinct phases characterizing the relationship between the loss decrease and the singular distribution:

\begin{enumerate}[leftmargin=*, topsep=2pt, itemsep=2pt]
    \item \textbf{Phase I (Fast Descent):} There exists a time $T_f \le T^*$ such that for all $t \le T_f$, the loss decrease satisfies:
    \begin{equation*}
        \Delta \mathcal{L}(t)  \ge \frac{3 D^2 \eta}{2(1+\sqrt{d})} = O(1).
    \end{equation*}

    \item \textbf{Phase II (Slow Descent):} There exists $T_s := \max\{T^*, T_\beta\}$ and $p > 2$ such that for all $t > T_s$, the singular distribution stabilizes within an $\varepsilon$-neighborhood:
    $\left\| \hat{\Sigma}_{t+1} - \hat{\Sigma}_{t} \right\|_F < \varepsilon.$
    Then, the loss decrement is polynomially bounded by the magnitude of SoSD:
    \begin{equation*}
        \Delta \mathcal{L}(t) \le O\!\left( \left\| \hat{\Sigma}_{t+1} - \hat{\Sigma}_{t} \right\|_F^{\,p} \right) = O(\varepsilon^{p}).
    \end{equation*}
\end{enumerate}
Here, $T^*$ is the global threshold time given by Theorem~\ref{thm:ssd_stability}.
\end{theorem}

The proof is provided in Appendix \ref{proof_thm2}.




\begin{figure*}[h]
  \vskip 0.2in
  \begin{center}
    \centerline{\includegraphics[width=2\columnwidth]{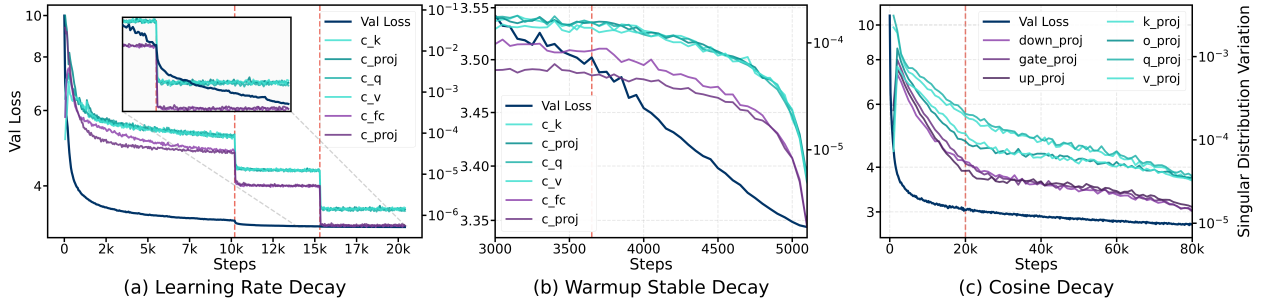}}
    \caption{
 \textbf{Singular Distribution Variation under different Learning Rate Schedules.} 
    In all subplots, the left y-axis represents the Validation Loss (dark blue line), and the right y-axis denotes the Singular Distribution Variation for various projection layers (colored lines).
    \textbf{(a) Learning Rate Decay on GPT-2 Small.} 
   The plot displays the training dynamics over 20k steps. Vertical red dashed lines at steps 10.2k and 15.3k indicate the points where the learning rate decays by a factor of 10. The inset provides a magnified view of the trajectories around the first decay point. 
    \textbf{(b) Warmup Stable Decay on GPT-2 Small.} 
    A focused view of the interval between steps 3000 and 5100. The vertical red dashed line at step 3650 marks the beginning of the WSD.
    \textbf{(c) Cosine Decay on LLaMA 0.5B.} 
    The plot spans 80k training steps, with a vertical red dashed line indicating the onset of the SoSD.
    }
    \label{fig_lrdecay_wsd}
  \end{center}
\end{figure*}

Theorem \ref{thm:two_phase} formalizes the \textbf{Two-Phase Dynamics} of pre-training by linking loss reduction to SD Variation:
\begin{itemize}[leftmargin=*, topsep=2pt, itemsep=2pt]
\item \textbf{Fast Descent Phase:} Initially, moderate parameter norms permit significant spectral shifts. In this regime, the loss decrease $\Delta \mathcal{L}(t)$ is lower-bounded by a constant order term ($O(1)$), enabling rapid optimization and active feature learning.
\item \textbf{Slow Descent Phase:} As the system transitions into SoSD (triggered by norm growth per Theorem \ref{thm:ssd_stability}), the loss reduction becomes kinetically constrained. The descent rate is no longer governed by gradient magnitude alone but is theoretically bounded by the variation of the singular distribution ($\Delta \mathcal{L} \le O(\varepsilon^p)$). Consequently, the stabilization of $\hat{\Sigma}_t$ acts as a bottleneck, forcing the loss into an asymptotic plateau.
\end{itemize}

\section{Interpreting Pre-training Strategies via SoSD Dynamics}
\label{section_interpreting}
In this section, we interpret the effectiveness of established pre-training strategies through the lens of SoSD. We posit that their empirical success stems from their ability to modulate the \textit{stability bound} derived in Theorem \ref{thm:ssd_stability}, thereby relaxing the kinetic constraints on loss reduction.

\begin{figure*}[t]
  \vskip 0.2in
  \begin{center}
    \centerline{\includegraphics[width=2\columnwidth]{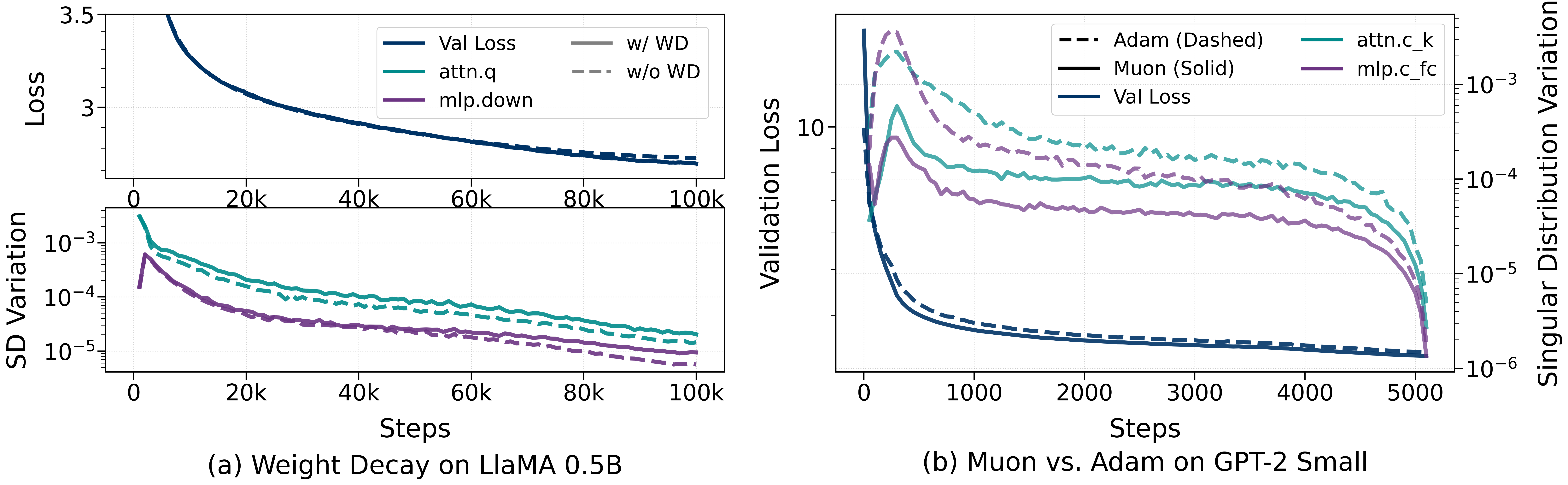}}
    \caption{
      \textbf{(a) Impact of Weight Decay on LlaMA 0.5B.} Comparison of training dynamics with (solid lines) and without (dashed lines) weight decay. The model with weight decay achieves lower validation loss, yet exhibits consistently higher SoSD values compared to the baseline without weight decay. \textbf{(b) Muon vs. Adam on GPT-2 Small.} Comparison between Muon and Adam optimizers. Muon demonstrates faster convergence and lower loss, while maintaining significantly lower SoSD values compared to Adam throughout the training process. For visual clarity, only representative matrices are displayed; others exhibit consistent behaviors.
    }
    \label{WD_and_muon}
  \end{center}
\end{figure*}

\subsection{Learning Rate Schedule}

We analyze three distinct learning rate (LR) schedules: Step-wise Decay, Warmup-Stable-Decay (WSD), and Cosine Decay, to demonstrate how scheduling $\eta_t$ directly controls the SoSD threshold.

\textbf{Step-wise Learning Rate Decay.}
Figure \ref{fig_lrdecay_wsd}(a) illustrates the dynamics of GPT-2 Small under a step-wise decay strategy. At steps 10.2k and 15.3k, the learning rate is reduced by a factor of 10. Coinciding with these discrete reductions, the SoSD metric ($\Delta \Sigma$) exhibits an immediate, sharp decline (see inset), synchronized with a renewed descent in validation loss.


Theoretically, this aligns with the stability bound $\varepsilon \propto \eta / \|W\|$ (Theorem \ref{thm:ssd_stability}). The abrupt reduction in $\eta$ tightens the permissible variation of the singular distribution. This forces the system out of its previous quasi-equilibrium state, allowing the model to resolve finer-grained features and settle into a new SoSD regime with a lower spectral noise floor, thereby minimizing the loss further.

\textbf{Continuous Annealing (WSD and Cosine Decay).}
In contrast to the discrete shifts of step decay, WSD (decay phase) and Cosine Decay employ continuous annealing.
\begin{itemize}[leftmargin=*, topsep=2pt, itemsep=2pt]
\item \textbf{Warmup-Stable-Decay (WSD):} Figure \ref{fig_lrdecay_wsd}(b) highlights the critical transition at step 3650. During the constant LR phase, the singular distribution oscillates around a fixed magnitude, and the loss stagnates. However, once the linear decay commences, the SoSD metric tracks the reduction of $\eta$, decreasing monotonically. This continuous relaxation of the stability bound prevents the model from locking into a static spectral configuration, enabling sustained loss improvement.

\item \textbf{Cosine Decay:} Similarly, under Cosine Decay (Figure \ref{fig_lrdecay_wsd}(c)), the smooth reduction of $\eta_t$ induces a gradual tightening of the stability bound $\varepsilon_t$. Consequently, the system maintains a \textit{dynamic} stability rather than a static equilibrium, driving a steady and persistent reduction in loss without the abrupt phase transitions observed in step-wise schedules.
\end{itemize}


\paragraph{Takeaway.}
From a spectral perspective, learning rate scheduling is equivalent to \textbf{scheduling the stability bound of SoSD}. By reducing $\eta$, these strategies progressively tighten the constraint on singular value variation, thereby mitigating the saturation effects imposed by SoSD and extending the effective optimization window.


\subsection{Weight Decay}

As shown in Figure~\ref{WD_and_muon}(a), the trajectory utilizing Weight Decay (WD) maintains consistently higher SoSD levels compared to the baseline. In terms of validation loss, both models exhibit similar performance during the initial 60k steps; however, beyond this phase, the WD model achieves a lower final loss compared to the non-WD baseline. This behavior aligns with our derived stability bound $\varepsilon \propto \eta / \|W\|$. By suppressing the growth of the weight norm $\|W\|$, WD relaxes the stability constraint (increasing $\varepsilon$) \citep{touvron2023llama}. According to Theorem~\ref{thm:two_phase}, this relaxed bound permits larger updates to the singular value distribution, effectively raising the theoretical upper bound for loss reduction. Consequently, by mitigating the growth rate of the weight norm, WD sustains the necessary optimization dynamics to achieve better convergence.

\paragraph{Takeaway} 
Our findings indicate that a lower SoSD does not essentially imply better performance. While the previous section showed that reducing the learning rate decreases SoSD alongside the loss, the Weight Decay experiment reveals the opposite: a higher SoSD coexists with a lower loss. Therefore, SoSD acts as a characterization of singular value dynamics rather than a direct performance metric. Its behavior serves as an indicator, signaling when the model has entered a regime of slow descent.

\subsection{Muon/Adam Optimizer}

Figure \ref{WD_and_muon}(b) presents the comparative training dynamics of Muon (solid lines) and Adam (dashed lines) under the Warmup-Stable-Decay (WSD) schedule \citep{jordan2024muon}. Regarding the Validation Loss, the Muon curve remains consistently below that of Adam, indicating that Muon achieves lower loss values at equivalent training steps. In terms of Singular Distribution Variation, both optimizers exhibit initial peaks; however, upon entering the SoSD phase, the variation produced by Muon is significantly lower than that of Adam, with a magnitude difference of approximately one order ($10^{-4}$ vs. $10^{-3}$). 


\paragraph{Takeaway} The emergence of SoSD in Muon-trained models confirms that the transition to Singular Distribution stability is a fundamental characteristic of Transformer pre-training, invariant to the choice of optimizer. However, the distinct scales of variation suggest that different optimizers navigate the spectral landscape with varying degrees of efficiency.





\section{Conclusion and Discussion}

In this work, we have investigated the intrinsic mechanisms governing the characteristic "fast-then-slow" convergence pattern in Large Language Model pre-training. By shifting the analytical lens to the temporal dynamics of spectral evolution, we identified the Stability of Singular Distribution (SoSD) phenomenon. Our empirical analysis across GPT-2 and LLaMA families demonstrates that the singular distribution stabilizes significantly earlier than the parameter matrices, with this onset of stability tightly synchronizing with the transition of the validation loss into the saturation regime.

Theoretically, we established a framework that proves the emergence of SoSD under conditions of non-degeneracy and gradient boundedness. We further revealed the dynamic coupling between loss descent and singular distribution variation, showing that the onset of SoSD theoretically bounds the subsequent loss reduction. Finally, we applied this spectral perspective to interpret standard pre-training strategies including learning rate schedule, weight decay and optimizer. We showed that learning rate schedules and weight decay effectively facilitate optimization by manipulating the stability bound $\varepsilon \propto \eta / \|W\|$, thereby allowing the model to overcome the spectral barrier to achieve lower training loss. Our work offers a novel mechanistic understanding of Transformer optimization trajectory, linking microscopic spectral behaviors to macroscopic training dynamics.

\section*{Impact Statement}
This study seeks to foster progress in deep learning research, with particular attention to deepening insights into and refining pre-training methodologies for Large Language Models (LLMs). While the potential broader societal implications of this work are acknowledged, no specific societal impacts have been identified at present that warrant explicit emphasis.

\section*{Acknowledgements}
This work was funded by the Strategic Priority Research Program of the Chinese Academy of Sciences (Grant No. XDB0680101), CAS Project for Young Scientists in Basic Research under Grant No. YSBR-034, the National Key Research and Development Program of China under Grants No. 2023YFA1011602, and Xiaomi Young Talents Program.


\bibliography{example_paper}
\bibliographystyle{icml2026}

\newpage
\appendix
\onecolumn

\section{Proofs of Auxiliary Lemmas}
\label{app:lemma-proofs}

\begin{lemma}\label{lem:1}
Let $a,b \in \mathbb{R}^n$, and define
$\alpha = \|a\|_1, \quad \beta = \|b\|_1,$
with $\alpha \neq \beta$.
Define the mapping $f : \mathbb{R}^n \to \mathbb{R}^n$ by
$f(x) = \frac{x}{\|x\|_1}.$
Then,
\begin{equation}\label{eq:l1-lipschitz}
\|f(a) - f(b)\|_2
\le
\frac{1+\sqrt{n}}{\min\{\alpha,\beta\}} \, \|a-b\|_2.
\end{equation}
\end{lemma}

\begin{proof}[Proof of lemma ~\ref{lem:1}]
We have
\begin{equation}
\begin{aligned}
\|f(a) - f(b)\|_2
&= \left\| \frac{a}{\alpha} - \frac{b}{\beta} \right\|_2 \\
&\le \left\| \frac{a}{\alpha} - \frac{b}{\alpha} \right\|_2
    + \left\| \frac{b}{\alpha} - \frac{b}{\beta} \right\|_2 ,
\end{aligned}
\end{equation}
where $\alpha = \|a\|_1$ and $\beta = \|b\|_1$.

For the first term, we have
\begin{equation}
\left\| \frac{a}{\alpha} - \frac{b}{\alpha} \right\|_2
= \frac{1}{\alpha} \|a - b\|_2 .
\end{equation}

For the second term, we obtain
\begin{equation}
\left\| \frac{b}{\alpha} - \frac{b}{\beta} \right\|_2
= \frac{|\beta - \alpha|}{\alpha \beta} \|b\|_2 .
\end{equation}

Note that
\begin{equation}
|\beta - \alpha|
= \big| \|b\|_1 - \|a\|_1 \big|
\le \|b - a\|_1
\le \sqrt{n} \|b - a\|_2 ,
\end{equation}
and
\begin{equation}
\|b\|_2 \le \|b\|_1 = \beta .
\end{equation}

Combining the above inequalities, we have
\begin{equation}
\begin{aligned}
\|f(a) - f(b)\|_2
&\le \frac{1}{\alpha} \|a - b\|_2
+ \frac{\beta - \alpha}{\alpha} \\
&\le \frac{1}{\alpha} \|a - b\|_2
+ \frac{\sqrt{n}}{\alpha} \|b - a\|_2 \\
&= \frac{1 + \sqrt{n}}{\alpha} \|b - a\|_2 .
\end{aligned}
\end{equation}

By symmetry (interchanging $a$ and $b$), we conclude that
\begin{equation}
\|f(a) - f(b)\|_2
\le \frac{1 + \sqrt{n}}{\min\{\alpha, \beta\}} \|b - a\|_2 .
\end{equation}
\end{proof}

\begin{lemma}[Mirsky's inequality: stability of singular values]
\label{lem:2}
Let $W_t, W_{t+1} \in \mathbb{R}^{m \times n}$, and let their singular value decompositions be
$W_t = U_t \Sigma_t V_t^*,
\quad
W_{t+1} = U_{t+1} \Sigma_{t+1} V_{t+1}^*.$
Then,
\begin{equation}\label{eq:mirsky}
\|\Sigma_{t+1} - \Sigma_t\|_F
\le
\|W_{t+1} - W_t\|_F.
\end{equation}
\end{lemma}

\begin{proof}[Proof of Lemma~\ref{lem:2}]
We have
\begin{equation}
\begin{aligned}
\|W_{t+1} - W_t\|_F^2
&= \operatorname{tr}\!\left( (W_{t+1} - W_t)^\top (W_{t+1} - W_t) \right) \\
&= \operatorname{tr}(W_{t+1}^\top W_{t+1})
   - 2 \operatorname{tr}(W_{t+1}^\top W_t)
   + \operatorname{tr}(W_t^\top W_t) \\
&= \sum_{i=1}^n \sigma_i^2(W_{t+1})
   + \sum_{i=1}^n \sigma_i^2(W_t)
   - 2 \operatorname{tr}(W_{t+1}^\top W_t).
\end{aligned}
\end{equation}

By the von Neumann trace inequality in Lemma~\ref{lem:3},
\begin{equation}
\operatorname{tr}(W_{t+1}^\top W_t)
\le \sum_{i=1}^n \sigma_i(W_{t+1}) \sigma_i(W_t).
\end{equation}

Therefore,
\begin{equation}
\begin{aligned}
\|W_{t+1} - W_t\|_F^2
&\ge
\|\Sigma_{t+1}\|^2_F
+ \|\Sigma_{t}\|^2_F
- 2 \sum_{i=1}^n \sigma_i(W_{t+1}) \sigma_i(W_t) \\
&=
\sum_{i=1}^n
\big( \sigma_i(W_{t+1}) - \sigma_i(W_t) \big)^2 \\
&=
\|\Sigma_{t+1} - \Sigma_t\|_F^2 .
\end{aligned}
\end{equation}

Hence,
\begin{equation}
\|\Sigma_{t+1} - \Sigma_t\|_F
\le
\|W_{t+1} - W_t\|_F .
\end{equation}
\end{proof}

\begin{lemma}[Von Neumann's inequality]
\label{lem:3}
For matrices $A, B \in \mathbb{R}^{m \times n}$, we have
\begin{equation}\label{eq:von-neumann}
\text{tr}(A^\top B) \leq \sum_{i=1}^{n} \sigma_i(A) \sigma_i(B),
\end{equation}
where $\sigma_{\min\{m,n\}}(A) \geq \dots  \geq \sigma_2(A)  \geq \sigma_1(A) \geq 0$ and $\sigma_{\min\{m,n\}}(B) \geq \dots \geq \sigma_2(B)  \geq \sigma_1(B)  \geq 0$ are the singular values of $A$ and $B$, respectively.
\end{lemma}

\begin{lemma}\label{lem:4}
Let $\{W_t\}_{t=0}^T$ denote the trajectory generated by gradient descent,
with $W_t \in \mathbb{R}^{d \times d}$.
\begin{equation}
  W_{t+1} = W_t - \eta \nabla \mathcal{L}(W_t)
= W_t - \eta G_t .  
\end{equation}

For each $t = 0,1,\dots,T$, let
$W_t = U_t \Sigma_t V_t^{\top}$
be the singular value decomposition of $W_t$, where $\Sigma_t$ denotes the diagonal matrix of singular values.
Then the following inequality holds:
\begin{equation}\label{eq:normalized-svd-lipschitz}
\left\|
\frac{\Sigma_{t+1}}{\operatorname{tr}(\Sigma_{t+1})}
-
\frac{\Sigma_t}{\operatorname{tr}(\Sigma_t)}
\right\|_F
\le
\frac{1+\sqrt{d}}{\min\!\left\{\operatorname{tr}(\Sigma_t),\,\operatorname{tr}(\Sigma_{t+1})\right\}}
\,
\|W_{t+1} - W_t\|_F
\le
\frac{1+\sqrt{d}}{\min\!\left\{\operatorname{tr}(\Sigma_t),\,\operatorname{tr}(\Sigma_{t+1})\right\}}
\,\eta \|G_t\|_F .
\end{equation}
\end{lemma}

\begin{proof}[Proof of lemma~\ref{lem:4}]
For each $W_t$, let
$W_t = U_t \Sigma_t V_t^{\top},$
and define
$p_t = \frac{\Sigma_t}{\operatorname{tr}(\Sigma_t)}.$

By Lemma~\ref{lem:1}, the normalization mapping
$f(X) = \frac{X}{\operatorname{tr}(X)}$
is Lipschitz continuous. Therefore, we have
\begin{equation}\label{eq:proof-normalized-bound}
\left\|
\hat{\Sigma}_{t+1} - \hat{\Sigma}_t
\right\|_F
\le
\frac{1+\sqrt{d}}{\min\!\left\{\operatorname{tr}(\Sigma_t),\,\operatorname{tr}(\Sigma_{t+1})\right\}}
\,
\|\Sigma_{t+1} - \Sigma_t\|_F .
\end{equation}

Using lemma~\ref{lem:2}, we further obtain
$\|\Sigma_{t+1} - \Sigma_t\|_F
\le
\|W_{t+1} - W_t\|_F .$

Combining the above inequalities yields
$\left\|
\hat{\Sigma}_{t+1} - \hat{\Sigma}_t
\right\|_F
\le
\frac{1+\sqrt{d}}{\min\!\left\{\operatorname{tr}(\Sigma_t),\,\operatorname{tr}(\Sigma_{t+1})\right\}}
\,
\|W_{t+1} - W_t\|_F .$

Finally, since $W_{t+1} - W_t = -\eta G_t$, we obtain
$\left\|
\frac{\Sigma_{t}}{tr(\Sigma_{t})} - \frac{\Sigma_{t+1}}{tr(\Sigma_{t+1})}
\right\|_F
\le
\frac{1+\sqrt{d}}{\min\!\left\{\operatorname{tr}(\Sigma_t),\,\operatorname{tr}(\Sigma_{t+1})\right\}}
\,\eta \|G_t\|_F .$
\end{proof}

\begin{lemma}[Lipschitz Gradient and Gradient Descent]\label{lem:5}
Assume the loss function $\mathcal{L}$ is $\beta$-smooth, i.e., there exists $\beta > 0$ such that for all
$\theta_1, \theta_2 \in \mathbb{R}^{d \times d}$,
$\|\nabla \mathcal{L}(\theta_1) - \nabla \mathcal{L}(\theta_2)\|
\le
\beta \|\theta_1 - \theta_2\|.$
Consider full-batch gradient descent with update rule
$\theta_{t+1} = \theta_t - \eta \nabla \mathcal{L}(\theta_t).$
Then the following inequalities hold:
\begin{enumerate}
\item
\begin{equation}
\eta\!\left(1 - \frac{\eta \beta}{2}\right)
\|\nabla \mathcal{L}(\theta_t)\|^2
\le
\mathcal{L}(\theta_t) - \mathcal{L}(\theta_{t+1}),
\end{equation}
\item
\begin{equation}
\eta\!\left(1 + \frac{\eta \beta}{2}\right)
\|\nabla \mathcal{L}(\theta_t)\|^2
\ge
\mathcal{L}(\theta_t) - \mathcal{L}(\theta_{t+1}).
\end{equation}
\end{enumerate}
\end{lemma}

\begin{lemma}\label{lem:6}
For any vector $u \in \mathbb{R}^m$, let
$j^* = \arg\max_j u_j,$
and define the gap as
$\operatorname{gap}(u) = u_{j^*} - \max_{j \neq j^*} u_j \ge 0.$
Let
$s = \operatorname{softmax}(u).$
Then the following hold:
\begin{enumerate}
\item
\begin{equation}
1 - s_{j^*} \le (m-1)\exp\!\left(-\operatorname{gap}(u)\right).
\end{equation}
\item
For all $j \neq j^*$,
\begin{equation}
s_j \le \exp\!\left(-\operatorname{gap}(u)\right).
\end{equation}
\end{enumerate}
\end{lemma}

\begin{proof}[Proof of Lemma~\ref{lem:6}]
For any $j \neq j^*$, by the definition of the gap we have
\begin{equation}\label{eq:gap-ineq}
u_j - u_{j^*} \le -\operatorname{gap}(u).
\end{equation}
Therefore,
\begin{equation}\label{eq:sj-gap}
s_j
=
\frac{\exp(u_j)}{\sum_{i=1}^m \exp(u_i)}
\le
\frac{\exp(u_j)}{\exp(u_{j^*})}
=
\exp(u_j - u_{j^*})
\le
\exp\!\left(-\operatorname{gap}(u)\right),
\end{equation}
where the last inequality follows from \eqref{eq:gap-ineq}.

As a consequence,
\begin{equation}\label{eq:one-minus-sjstar}
1 - s_{j^*}
=
\sum_{j \neq j^*} s_j
\le
(m-1)\exp\!\left(-\operatorname{gap}(u)\right).
\end{equation}
\end{proof}

\begin{lemma}\label{lem:7}
Let $a = \operatorname{softmax}(u)$ with $u \in \mathbb{R}^m$.  
The Jacobian matrix of $a$ is given by
\begin{equation}\label{eq:softmax-jacobian}
J(a) = \operatorname{diag}(a) - a a^\top .
\end{equation}
Then $J(a)$ is positive semidefinite, and
\begin{equation}\label{eq:jacobian-norm}
\|J(a)\|_2 \le \operatorname{tr}(J(a)) = 1 - \|a\|_2^2 \le 2(1 - a_{\max}),
\end{equation}
where
\begin{equation}\label{eq:amax-def}
a_{\max} = \max_j a_j .
\end{equation}
\end{lemma}

\begin{proof}[proof of lemma~\ref{lem:7}]
(1) \text{To show that } $J(a) \succeq 0$.

For any $\mathbf{x} \in \mathbb{R}^m$, we have
\begin{equation}
\begin{aligned}
\mathbf{x}^T J(a) \mathbf{x}
&= \mathbf{x}^T \operatorname{diag}(a) \mathbf{x}
   - \mathbf{x}^T (aa^T) \mathbf{x} \\
&= \sum_i a_i x_i^2 - \sum_i a_i^2 x_i^2 \\
&= \sum_i (a_i - a_i^2) x_i^2 \\
&\ge 0, \qquad \text{since } a_i \le 1 .
\end{aligned}
\end{equation}
Thus, $J(a) \succeq 0$.

(2) Since $J(a) \succeq 0$, we have
$\|J(a)\|_2 = \lambda_{\max}(J(a))
\le \sum_i \lambda_i(J(a))
= \operatorname{tr}(J(a)).$

First, we calculate the trace of $J(a)$:
\begin{equation}
\text{tr}(J(a)) = \sum_i a_i - \sum_i a_i^2 = 1 - \|a\|_2^2
\end{equation}

Next, we use the inequality $\|a\|_2^2 \geq a_{\max}^2$, which gives:
\begin{equation}
1 - \|a\|_2^2 \leq 1 - a_{\max}^2 = (1 + a_{\max})(1 - a_{\max}) \leq 2(1 - a_{\max}).
\end{equation}

Thus, we have:
$\|J(a)\|_2 \leq \text{tr}(J(a)) \leq 2(1 - a_{\max}).$
\end{proof}

\begin{lemma}
\label{lem:8}
For models of the form~\ref{model arch}, under small perturbations, the following approximation holds:
\begin{equation}
H(x) \doteq (A_0 + J_1 M) W W^T ,
\end{equation}
where
\begin{equation}
A_0 = \frac{1}{n} \mathbf{1}\mathbf{1}^T,
\qquad
J_1 = \frac{1}{n}\left(I - \frac{1}{n}\mathbf{1}\mathbf{1}^T\right).
\end{equation}
\end{lemma}

\begin{proof}[proof of ~\ref{lem:8}]
Consider the softmax function
\begin{equation}
\begin{aligned}
A(x)
&= \operatorname{softmax}(M)
 = \left[ \frac{\exp(M_{ij})}{\sum_{k=1}^{n} \exp(M_{ik})} \right]_{ij} \\
&\approx \left[ \frac{1 + M_{ij}}{n + \sum_{k=1}^{n} M_{ik}} \right]_{ij} \\
&\approx \left[ \frac{1 + M_{ij}}{n\!\left(1 + \frac{1}{n}\sum_{k=1}^{n} M_{ik}\right)} \right]_{ij} \\
&\approx \frac{1}{n}
\left[
(1 + M_{ij})
\left(1 - \frac{1}{n}\sum_{k=1}^{n} M_{ik}\right)
\right]_{ij} \\
&\approx
\left[
\frac{1}{n}
\left(
1 + M_{ij}
- \frac{1}{n}\sum_{k=1}^{n} M_{ik}
\right)
+ O(\sigma^4)
\right]_{ij} \\
&\approx
\left[
\frac{1}{n}
+ \frac{1}{n} M_{ij}
- \frac{1}{n^2}\sum_{k=1}^{n} M_{ik}
\right]_{ij} \\
&\approx A_0 + J_1 M, 
\end{aligned}
\end{equation}
where
$A_0 = \frac{1}{n}\mathbf{1}\mathbf{1}^T,
\quad
J_1 = \frac{1}{n}\left(I - \frac{1}{n}\mathbf{1}\mathbf{1}^T\right).$
\end{proof}

\section{Proof of Theorem~\ref{thm:ssd_stability}}
\label{proof_for_thm1}
\begin{proof}[Proof of Theorem~\ref{thm:ssd_stability}.]

\noindent

\textbf{Part 1.} 
First, we analyze $\|G_t^V\|_F^2$.

For each sample $i$, we define the softmax probability assigned to
the ground-truth label $y_i$ as
\begin{equation}
\label{eq:softmax_prob}
P_{i,y_i}
=
\frac{\exp(Z_{i,y_i})}
{\sum_{k=1}^C \exp(Z_{i,k})}.
\end{equation}

Taking the logarithm of \eqref{eq:softmax_prob}, we obtain
\begin{equation}
\label{eq:log_softmax}
\log P_{i,y_i}
=
Z_{i,y_i}
-
\log \sum_{k=1}^C \exp(Z_{i,k}).
\end{equation}

Therefore, the empirical loss function can be written as
\begin{equation}
\label{eq:empirical_loss}
\mathcal{L}(\theta)
=
\frac{1}{n}
\sum_{i=1}^n
\left(
\log \sum_{k=1}^C \exp(Z_{i,k})
-
Z_{i,y_i}
\right).
\end{equation}

By direct computation, the gradient of the loss with respect to $z_i$ is given by
\begin{equation}
\label{eq:loss_gradient}
\frac{\partial \mathcal{L}}{\partial z_i}
=
\frac{1}{n}
\left( p_i - y_i \right).
\end{equation}

Stacking all samples together, the gradient of the loss with respect to
the logit matrix $Z$ is given by
\begin{equation}
\label{eq:GZ_stack}
\frac{\partial \mathcal{L}}{\partial Z}
=
\frac{1}{n}(P - Y).
\end{equation}

For convenience, we define
\begin{equation}
\label{eq:GZ_def}
G_Z
:=
\frac{\partial \mathcal{L}}{\partial Z}
=
\frac{1}{n}(P - Y).
\end{equation}

The gradient with respect to the hidden representation $H$ then satisfies
\begin{equation}
\label{eq:GH_def}
G_H
:=
\frac{\partial \mathcal{L}}{\partial H}
=
G_Z W_C^{\top}.
\end{equation}

Recall that $H = A V$. Applying the chain rule, we obtain the gradients
with respect to $V$ and $A$ as
\begin{equation}
\label{eq:GV_GA}
G_V
:=
\frac{\partial \mathcal{L}}{\partial V}
=
A^{\top} G_H,
\qquad
G_A
:=
\frac{\partial \mathcal{L}}{\partial A}
=
G_H V^{\top}.
\end{equation}

Finally, using $V = X W_V$, the gradient with respect to $W_V$ is given by
\begin{equation}
\label{eq:GWV}
G_{W_V}
:=
G^V
=
\frac{\partial \mathcal{L}}{\partial W_V}
=
X^{\top} G_V
=
X^{\top} A^{\top} G_H.
\end{equation}
Under the small initialization regime, as stated in Lemma~\ref{lem:8}, we have
\begin{equation}
\label{eq:A_approx}
\begin{aligned}
A 
&\approx \frac{1}{n}\mathbf{1}\mathbf{1}^{\top}
+ \frac{1}{n} M
- \frac{1}{n^2}\mathbf{1}\mathbf{1}^{\top} M \\
&\approx \frac{1}{n}\mathbf{1}\mathbf{1}^{\top}
+ \frac{1}{n}
\left(
I - \frac{1}{n}\mathbf{1}\mathbf{1}^{\top}
\right) M \\
&\approx A_0 + J_1 M .
\end{aligned}
\end{equation}
Here,
\begin{equation}
\label{eq:A0_J1_def}
A_0 := \frac{1}{n}\mathbf{1}\mathbf{1}^{\top},
\qquad
J_1 := \frac{1}{n}\left(I - \frac{1}{n}\mathbf{1}\mathbf{1}^{\top}\right).
\end{equation}

Furthermore, for $G^V$ we have
\begin{equation}
\label{eq:GV_norm_form}
\|G^V\|_F
= \| X^{\top} A^{\top} G_H \|_F
= \left[
\operatorname{tr}\!\left(
G_H^{\top} A X X^{\top} A^{\top} G_H
\right)
\right]^{1/2}.
\end{equation}

Assume $X \in \mathbb{R}^{n \times d}$ and that $X X^{\top} \in \mathbb{R}^{n \times n}$ is full-rank. Then, it holds that
\begin{equation}
\label{eq:GV_lower_bound}
\|G^V\|_F
\ge
\lambda_{\min}(X)
\left[
\operatorname{tr}\!\left(
G_H^{\top} A A^{\top} G_H
\right)
\right]^{1/2}.
\end{equation}

Taking the approximation $A \approx A_0$ simplifies the analysis, where
\begin{equation}
\label{eq:A0_def}
A_0 := \frac{1}{n}\mathbf{1}\mathbf{1}^{\top}.
\end{equation}

Moreover, we observe that
\begin{equation}
\label{eq:A0_property}
A_0 A_0^{\top}
= \frac{1}{n}\mathbf{1}\mathbf{1}^{\top}
\frac{1}{n}\mathbf{1}\mathbf{1}^{\top}
= \frac{1}{n}\mathbf{1}\mathbf{1}^{\top}
= A_0.
\end{equation}

Thus, we obtain the following lower bound for the Frobenius norm of $G^V$:
\begin{equation}
\label{eq:GV_lower_bound_approx}
\|G^V\|_F
\gtrsim
\lambda_{\min}(X)
\left[
\operatorname{tr}\!\left(
G_H^{\top} A_0 G_H
\right)
\right]^{1/2}
\gtrsim
\frac{1}{\sqrt{n}}\,\lambda_{\min}(X)\,\|G_H\|_F.
\end{equation}

Finally, we define the constant
\begin{equation}
\label{eq:CV_def}
C_V := \frac{1}{\sqrt{n}}\,\lambda_{\min}(X)\,\|G_H\|_F.
\end{equation}

For notational simplicity, we define
\begin{equation}
\label{eq:weight_norm_defs}
\|W_V\|_F = v .
\end{equation}

From the gradient relations derived above, we obtain the following differential inequality:
\begin{equation}
\label{eq:ode_system_1}
\dot{v} \gtrsim C_V,
\end{equation}

Solving the differential inequality in \eqref{eq:ode_system_1}, we obtain
\begin{equation}
\label{eq:ode_solution_1}
v(t)\ge
C_V t + v(0).
\end{equation}

By Lemma~\ref{lem:4}, The matrices $W_V$ can be decomposed via SVD. With the help of Assumption \ref{ass:increasing_norm}, We have
$$\begin{aligned}
\left\|
\frac{\Sigma_{t+1}}{\operatorname{tr}(\Sigma_{t+1})}
-
\frac{\Sigma_t}{\operatorname{tr}(\Sigma_t)}
\right\|_F
&\le
\frac{1+\sqrt{d}}{\min\{\operatorname{tr}(\Sigma_t),\, \operatorname{tr}(\Sigma_{t+1})\}}
\,\|W_{t+1}-W_t\|_F \\
&\le
\frac{1+\sqrt{d}}{\min\{\operatorname{tr}(\Sigma_t),\, \operatorname{tr}(\Sigma_{t+1})\}}
\,\eta \|G_t\|_F \\
&\le
\frac{1+\sqrt{d}}{\operatorname{tr}(\Sigma_{t})}
\,\eta G .
\end{aligned}$$

For stability bound $\varepsilon(W_V)=O(\frac{\eta}{\|W_V\|})$ , we have
\begin{equation}
\label{eq:sigma_eps_bound}
\left\|
\frac{\Sigma_{t+1}}{\operatorname{tr}(\Sigma_{t+1})}
-
\frac{\Sigma_t}{\operatorname{tr}(\Sigma_t)}
\right\|_F
\le
\frac{1+\sqrt{d}}{\operatorname{tr}(\Sigma_t)} \, \eta G
<
\varepsilon(W) .
\end{equation}

Hence, using the relation between the Frobenius norm and the nuclear norm,
it follows that
\begin{equation}
\label{eq:Wt_norm_lower}
\|W_t\|_F
=
\|\Sigma_t\|_F
\ge
\frac{\|\Sigma_t\|_*}{\sqrt{d}}
=
\frac{\|W_t\|_*}{\sqrt{d}}
>
\frac{(1+\sqrt{d})\,\eta G}{\sqrt{d}\,\epsilon} .
\end{equation}

As a result, for the query, key, and value projection matrice $W_V$, a sufficient condition ensuring
\eqref{eq:sigma_eps_bound} is that their Frobenius norms satisfy
\begin{equation}
\label{eq:weight_sufficient_condition_1}
v(t) > \dfrac{(1+\sqrt{d})\,\eta G}{\sqrt{d}\,\epsilon}.
\end{equation}

By solving the above inequality using the growth lower bounds of $v(t)$, we obtain the corresponding hitting times
\begin{equation}
\label{eq:Tv_def}
T_V
=
O\left(\frac{(1+\sqrt{d})\,\eta G - \varepsilon v(0)}{\varepsilon\, C_V\sqrt{d}}\right).
\end{equation}

\noindent\textbf{Part 2.} 
Next, we separately analyze
$\|G_t^Q\|_F^2$ and $\|G_t^K\|_F^2$.

Substituting \eqref{eq:A_approx} into the gradients yields
\begin{equation}
\label{eq:GWQ_GWK}
\begin{aligned}
G_{W_Q} :=G^{Q}
&= \frac{1}{\sqrt{d}} \, X^{\top} J_1^{\top} G_H V^{\top} K, \\
G_{W_K} := G^{K}
&= \frac{1}{\sqrt{d}} \, X^{\top} V G_H^{\top} J_1 Q .
\end{aligned}
\end{equation}

Then we consider the Frobenius norms of $G^{Q}$ and $G^{K}$.

We first analyze $\|G^{Q}\|_{F}$. By definition, we have
\begin{equation}
\label{eq:GQ_norm_lower_bound}
\begin{aligned}
\|G^{Q}\|_{F}
&= \left\| \frac{1}{\sqrt{d}} X J_1^{\top} G_H V^{\top} \right\|_{F} \\
&= \frac{1}{\sqrt{d}} \left\| X J_1^{\top} G_H V^{\top} \right\|_{F} \\
&= \frac{1}{\sqrt{d}}
\left[
\operatorname{tr}
\!\left(
K^{\top} V G_H^{\top} J_1
X X^{\top}
J_1^{\top} G_H V K
\right)
\right]^{1/2} \\
&\ge
\frac{1}{\sqrt{d}} \,
\lambda_{\min}(X)\,
\left[
\operatorname{tr}
\!\left(
K^{\top} V G_H^{\top}
J_1 J_1^{\top}
G_H V K
\right)
\right]^{1/2}.
\end{aligned}
\end{equation}

Recall that
\begin{equation}
\label{eq:J1_def}
J_1
=
\frac{1}{n}
\left(
I - \frac{1}{n}\mathbf{1}\mathbf{1}^{\top}
\right).
\end{equation}

Moreover, we have
\begin{equation}
\label{eq:J1_square}
\begin{aligned}
J_1 J_1^{\top}
&=
\frac{1}{n^2}
\left(
I - \frac{1}{n}\mathbf{1}\mathbf{1}^{\top}
\right)
\left(
I - \frac{1}{n}\mathbf{1}\mathbf{1}^{\top}
\right) \\
&=
\frac{1}{n^2}
\left(
I - \frac{1}{n}\mathbf{1}\mathbf{1}^{\top}
\right).
\end{aligned}
\end{equation}

Substituting this identity into the expression yields
\begin{equation}
\label{eq:GQ_split_terms}
\begin{aligned}
\|G^Q\|_F
&\ge
\frac{1}{n\sqrt{d}}\,
\lambda_{\min}(X)\,
\underbrace{
\left[
\operatorname{tr}
\!\left(
K^{\top} V G_H^{\top} G_H V^{\top} K
\right)
\right]^{1/2}
}_{\text{(I)}} \\
&\quad
-
\frac{1}{n^{3/2}\sqrt{d}}\,
\lambda_{\min}(X)\,
\underbrace{
\left[
\operatorname{tr}
\!\left(
K^{\top} V G_H^{\top}
\mathbf{1}\mathbf{1}^{\top}
G_H V^{\top} K
\right)
\right]^{1/2}
}_{\text{(II)}} .
\end{aligned}
\end{equation}

We decompose the right-hand side of \eqref{eq:GQ_split_terms} into two components,
denoted by \textnormal{(I)} and \textnormal{(II)}, respectively.

We first bound term \textnormal{(I)}.
\begin{equation}
\label{eq:term_I_bound}
\begin{aligned}
\textnormal{(I)}
&=
\left[
\operatorname{tr}
\!\left(
G_H V^{\top} K^{\top} K V G_H^{\top}
\right)
\right]^{1/2} \\
&\ge
\lambda_{\min}(X)\,
\sigma_{\min}(W_K)\,
\left[
\operatorname{tr}
\!\left(
G_H V^{\top} V G_H^{\top}
\right)
\right]^{1/2} \\
&\ge
\lambda_{\min}^2(X)\,
\sigma_{\min}(W_K)\,
\sigma_{\min}(W_V)\,
\|G_H\|_F .
\end{aligned}
\end{equation}

We next bound term \textnormal{(II)}.
\begin{equation}
\label{eq:term_II_bound}
\begin{aligned}
\textnormal{(II)}
&=
\left[
\operatorname{tr}
\!\left(
K V G_H^{\top}
\mathbf{1}\mathbf{1}^{\top}
G_H V^{\top} K
\right)
\right]^{1/2} \\
&=
\left[
\operatorname{tr}
\!\left(
\mathbf{1}^{\top}
G_H V^{\top} K K V G_H^{\top}
\mathbf{1}
\right)
\right]^{1/2} \\
&\ge
\lambda_{\min}^2(X)\,
\sigma_{\min}(W_K)\,
\sigma_{\min}(W_V)\,
\left[
\operatorname{tr}
\!\left(
\mathbf{1}^{\top}
G_H G_H^{\top}
\mathbf{1}
\right)
\right]^{1/2} \\
&\ge
\lambda_{\min}^2(X)\,
\sigma_{\min}(W_K)\,
\sigma_{\min}(W_V)\,
\|G_H\|_F .
\end{aligned}
\end{equation}

Combining the bounds on \textnormal{(I)} and \textnormal{(II)} and substituting them into
\eqref{eq:GQ_split_terms}, we obtain
\begin{equation}
\label{eq:GQ_final_bound_raw}
\|G^Q\|_F
\ge
\left(\frac{1}{n}-\frac{1}{n^{3/2}}\right)
\frac{1}{\sqrt{d}}\,
\lambda_{\min}^3(X)\,
\|G_H\|_F\,
\sigma_{\min}(W_K)\,
\sigma_{\min}(W_V).
\end{equation}

Under the numerical non-degeneracy assumption in Assumption~\ref{ass:regularity},
\begin{equation}
\label{eq:GQ_final_bound_kappa}
\begin{aligned}
\|G^Q\|_F
&\ge
\frac{\sqrt{n}-1}{n^{3/2}\sqrt{d}}\,
\lambda_{\min}^3\,
\|G_H\|_F\,
\sigma_{\max}(W_K)\,
\sigma_{\max}(W_V)\,
\frac{1}{\kappa_K \kappa_V} \\
&\ge
\frac{\sqrt{n}-1}{n^{3/2} d^{5/2}}\,
\lambda_{\min}^3\,
\|G_H\|_F\,
\frac{1}{\kappa_K \kappa_V}\,
\|W_K\|_F\,
\|W_V\|_F .
\end{aligned}
\end{equation}

We define the constant
\begin{equation}
\label{eq:CQ_def}
C_Q
:=
\frac{(\sqrt{n}-1)\,\lambda_{\min}^3\,\|G_H\|_F}
     {n^{3/2} d^{5/2} \kappa_K \kappa_V}.
\end{equation}

Therefore, we arrive at the compact form
\begin{equation}
\label{eq:GQ_final_compact}
\|G^Q\|_F
\ge
C_Q \,
\|W_K\|_F\,
\|W_V\|_F .
\end{equation}

By symmetry, an analogous bound holds for the key projection, namely,
\begin{equation}
\label{eq:GK_bound}
\| G^K \|_F
\ge
C_K \,
\| W_Q \|_F \,
\| W_V \|_F ,
\end{equation}
where the constant $C_K$ is given by
\begin{equation}
\label{eq:CK_def}
C_K
=
\frac{(\sqrt{n}-1)\, \lambda_{\min}^3 \, \| G_H \|_F}
     {n^{3/2} \, d^{5/2}\, \kappa_Q \kappa_V}.
\end{equation}

For notational simplicity, we define
\begin{equation}
\label{eq:weight_norm_defs}
\|W_Q\|_F = q,
\qquad
\|W_K\|_F = k.
\end{equation}

From the gradient relations derived above, we obtain the following system
of differential inequalities:
\begin{equation}
\label{eq:ode_system}
\begin{cases}
\dot{q} \;\ge\; C_Q \, k v,\\
\dot{k} \;\ge\; C_K \, q v .
\end{cases}
\end{equation}

Due to the symmetry between $q$ and $k$, without loss of generality
we assume
\begin{equation}
\label{eq:symmetry_assumption}
k(0) = q(0),
\qquad
C_Q = C_K =C_M.
\end{equation}

Solving the system of differential inequalities in \eqref{eq:ode_system}
under the assumption \eqref{eq:symmetry_assumption} and substituting
\eqref{eq:ode_solution_1} into the resulting expressions, we obtain
\begin{equation}
\label{eq:ode_solution}
\begin{aligned}
q(t)
&\ge
q(0)\,
\exp\!\left(
\tfrac{1}{2} C_M C_V t^2 + C_M v(0)\, t
\right), \\
k(t)
&\ge
k(0)\,
\exp\!\left(
\tfrac{1}{2} C_M C_V t^2 + C_M v(0)\, t
\right).
\end{aligned}
\end{equation}

Hence, using the inequality~\eqref{eq:sigma_eps_bound} and the relation between the Frobenius norm and the nuclear norm, for stability bound $\varepsilon=\varepsilon(W)=O(\frac{\eta}{\|W\|}), \quad W \in \{W_Q, W_K\}$,
it follows that
\begin{equation}
\label{eq:Wt_norm_lower}
\|W_t\|_F
=
\|\Sigma_t\|_F
\ge
\frac{\|\Sigma_t\|_*}{\sqrt{d}}
=
\frac{\|W_t\|_*}{\sqrt{d}}
>
\frac{(1+\sqrt{d})\,\eta G}{\sqrt{d}\,\varepsilon} .
\end{equation}

As a result, for the query, key, and value projection matrices
$W_Q$, $W_K$, and $W_V$, a sufficient condition ensuring
\eqref{eq:sigma_eps_bound} is that their Frobenius norms satisfy
\begin{equation}
\label{eq:weight_sufficient_condition}
\begin{cases}
q(t) > \dfrac{(1+\sqrt{d})\,\eta G}{\sqrt{d}\,\varepsilon}, \\
v(t) > \dfrac{(1+\sqrt{d})\,\eta G}{\sqrt{d}\,\varepsilon}.
\end{cases}
\end{equation}

By solving the above inequalities using the growth lower bounds of
$q(t)$ and $v(t)$ and setting $C_M v(0) = C_0$, we obtain the corresponding hitting times
\begin{equation}
\label{eq:Tq_def}
T_{QK}
=
\frac{1}{C_M C_V}
\left(
- C_M v(0)
+
\sqrt{
\bigl(C_M v(0)\bigr)^2
+
2 C_M v(0) \Lambda(\varepsilon)
}
\right) = O\Bigg( \bigg( C_0^2 + 2 C_0 \Lambda(\varepsilon) \eta G \bigg)^{1/2} \Bigg),
\end{equation}
 where $\Lambda(\varepsilon)=\ln(1+\sqrt{d})-\ln(\varepsilon\, q(0)\sqrt{d})$.

Therefore, for stability bound $\varepsilon(W)=O(\frac{1}{\|W\|}), W \in \{W_Q,W_K,W_V\}$ , there exists a constant $C > 0$
such that for
$t > T^* = C\max \left\{ 
        \frac{(1+\sqrt{d})\,\eta G}{\varepsilon\, C_V\sqrt{d}}, 
        \sqrt{C_0^2 + 2 C_0 \Lambda(\varepsilon) \eta G} 
    \right\}$,
the normalized singular value matrices satisfy
\begin{equation}
\label{eq:sigma_convergence_final}
\left\|
\frac{\Sigma_{t+1}}{\operatorname{tr}(\Sigma_{t+1})}
-
\frac{\Sigma_t}{\operatorname{tr}(\Sigma_t)}
\right\|_F
< \varepsilon.
\end{equation}

\end{proof}

\section{Proof of Theorem~\ref{thm:two_phase}}
\label{proof_thm2}
\begin{proof}[Proof of Theorem~\ref{thm:two_phase}.]

By Lemma~\ref{lem:4}, for the parameter sequence
$\{W_{\bullet}(t)\}_{t=0}^T$, where $\bullet \in \{Q,K,V\}$,
and their corresponding singular value matrices $\Sigma_t$,
the following bound holds:
\begin{equation}
\label{eq:lemma4_bound_proof}
\left\|
\frac{\Sigma_{t+1}}{\operatorname{tr}(\Sigma_{t+1})}
-
\frac{\Sigma_t}{\operatorname{tr}(\Sigma_t)}
\right\|_F
\le
\frac{1+\sqrt{d}}
{\min\{\operatorname{tr}(\Sigma_t),\,\operatorname{tr}(\Sigma_{t+1})\}}
\,\eta\,\|G_t\|_F .
\end{equation}

For notational convenience, define
\begin{equation}
\label{eq:tau_def}
\tau_t := \operatorname{tr}(\Sigma_t),
\end{equation}
and introduce the normalized singular distribution difference
\begin{equation}
\label{eq:delta_def}
\delta(t)
:=
\left\|
\frac{\Sigma_{t+1}}{\operatorname{tr}(\Sigma_{t+1})}
-
\frac{\Sigma_t}{\operatorname{tr}(\Sigma_t)}
\right\|_F .
\end{equation}

Applying \eqref{eq:lemma4_bound_proof} to the query, key, and value
projection matrices $W_Q$, $W_K$, and $W_V$, respectively,
we obtain
\begin{equation}
\label{eq:delta_Q_bound}
\delta_Q(t)
\le
\frac{(1+\sqrt{d})\,\eta}{\tau_t^{Q}}
\,\|G_t^{Q}\|_F ,
\end{equation}
\begin{equation}
\label{eq:delta_K_bound}
\delta_K(t)
\le
\frac{(1+\sqrt{d})\,\eta}{\tau_t^{K}}
\,\|G_t^{K}\|_F ,
\end{equation}
and
\begin{equation}
\label{eq:delta_V_bound}
\delta_V(t)
\le
\frac{(1+\sqrt{d})\,\eta}{\tau_t^{V}}
\,\|G_t^{V}\|_F .
\end{equation}

Let
\begin{equation}
\label{eq:C0_def}
C_0 := (1+\sqrt{d})\,\eta .
\end{equation}
Then, by rearranging the bounds in
\eqref{eq:delta_Q_bound}–\eqref{eq:delta_V_bound},
we obtain
\begin{equation}
\label{eq:GQ_lower}
\|G_t^{Q}\|_F
\ge
\frac{1}{C_0}\,
\tau_t^{Q}\,
\delta_Q(t),
\qquad
\|G_t^{K}\|_F
\ge
\frac{1}{C_0}\,
\tau_t^{K}\,
\delta_K(t),
\end{equation}
and
\begin{equation}
\label{eq:GV_lower}
\|G_t^{V}\|_F
\ge
\frac{1}{C_0}\,
\tau_t^{V}\,
\delta_V(t).
\end{equation}

Let
\begin{equation}
\label{eq:Theta_def}
\theta := \{W_Q,\, W_K,\, W_V\}
\end{equation}
denote the collection of model parameters.
Then the squared Frobenius norm of the full gradient satisfies
\begin{equation}
\label{eq:grad_theta_norm}
\|\nabla_{\theta} \mathcal{L}(\theta)\|_F^2
=
\|G_t^{Q}\|_F^2
+
\|G_t^{K}\|_F^2
+
\|G_t^{V}\|_F^2 .
\end{equation}

By Lemma~\ref{lem:5}, the one-step decrease of the loss obeys
\begin{equation}
\label{eq:loss_decrease_lemma}
\Delta \mathcal{L}(t)
\triangleq
\mathcal{L}(t) - \mathcal{L}(t+1)
\ge
\eta\!\left(1 - \frac{\eta\beta}{2}\right)
\|\nabla_{\theta} \mathcal{L}(\theta)\|_F^2 .
\end{equation}

Substituting \eqref{eq:grad_theta_norm} into
\eqref{eq:loss_decrease_lemma} yields
\begin{equation}
\label{eq:loss_decrease_expand}
\Delta \mathcal{L}(t)
\ge
\eta\!\left(1 - \frac{\eta\beta}{2}\right)
\Bigl(
\|G_t^{Q}\|_F^2
+
\|G_t^{K}\|_F^2
+
\|G_t^{V}\|_F^2
\Bigr).
\end{equation}

Finally, applying the lower bounds
\eqref{eq:GQ_lower}–\eqref{eq:GV_lower}, we obtain
\begin{equation}
\label{eq:loss_decrease_delta}
\Delta \mathcal{L}(t)
\ge
\frac{\eta}{C_0}
\left(1 - \frac{\eta\beta}{2}\right)
\Bigl(
\bigl(\tau_t^{Q}\delta_Q(t)\bigr)^2
+
\bigl(\tau_t^{K}\delta_K(t)\bigr)^2
+
\bigl(\tau_t^{V}\delta_V(t)\bigr)^2
\Bigr).
\end{equation}

In the early stage of training, there exists a constant $D > 0$ such that
\begin{equation}
\label{eq:early_stage_lower}
\max\{\tau^{\bullet}_t\delta_{\bullet}(t),\bullet \in \{Q,K,V\}\}\ge D.
\end{equation}

Combining \eqref{eq:early_stage_lower} and \eqref{eq:C0_def} with
\eqref{eq:loss_decrease_delta}, we obtain
\begin{equation}
\label{eq:loss_decrease_early_correct}
\Delta \mathcal{L}(t)
\ge
\frac{3D^2}{1+\sqrt{d}}
\left(1-\frac{\eta\beta}{2}\right).
\end{equation}

Choosing the step size $\eta = \frac{1}{\beta}$, it follows that
\begin{equation}
\label{eq:loss_decrease_constant_correct}
\Delta \mathcal{L}(t)
\ge
\frac{3D^2}{2(1+\sqrt{d})}.
\end{equation}

Consequently,
\begin{equation}
\label{eq:loss_O1_correct}
\Delta \mathcal{L}(t) = O(1),
\end{equation}
that is, the loss decreases by a constant amount in the early stage
of training.

By Lemma~\ref{lem:5}, the one-step decrease of the loss satisfies
\begin{equation}
\label{eq:loss_upper_bound}
\Delta \mathcal{L}(t)
\triangleq
\mathcal{L}(t) - \mathcal{L}(t+1)
\le
\eta\!\left(1 + \frac{\eta\beta}{2}\right)
\|\nabla_{\theta} \mathcal{L}(\theta)\|_F^2 .
\end{equation}

By Theorem~\ref{thm:ssd_stability}, the gradient with respect to the value projection
matrix admits the representation
\begin{equation}
\label{eq:GV_representation}
G^V
=
X^{\top} A^{\top} G_H .
\end{equation}

Recall that $A = \operatorname{softmax}(M)$.
Let $a_i \in \mathbb{R}^n$ denote the $i$-th row of $A$, i.e.,
\begin{equation}
\label{eq:ai_softmax}
a_i
=
\operatorname{softmax}(M_i).
\end{equation}
The Jacobian of $a_i$ with respect to $M_i$ is given by
\begin{equation}
\label{eq:softmax_jacobian}
J(a_i)
=
\frac{\partial a_i}{\partial M_i}
=
\operatorname{diag}(a_i) - a_i a_i^{\top}.
\end{equation}

Therefore, the gradient with respect to the attention logits $M_i$
can be written as
\begin{equation}
\label{eq:GM_i_expression}
G_{M,i}
=
\frac{\partial L}{\partial M_i}
=
J(a_i)^{\top} G_{A,i}.
\end{equation}

Stacking all rows together, the gradient with respect to the attention
logits $M$ can be written as
\begin{equation}
\label{eq:GM_stack}
G_M
=
\frac{\partial \mathcal{L}}{\partial M}
=
\bigl[ J(a_i)^{\top} G_{A,i} \bigr]_{i=1}^n .
\end{equation}

Recall that
\begin{equation}
\label{eq:M_QK_def}
M = \frac{QK^{\top}}{\sqrt{d}},
\qquad
Q = X W_Q,
\qquad
K = X W_K .
\end{equation}

Hence, the gradients with respect to the projection matrices
$W_Q$ and $W_K$ are given by
\begin{equation}
\label{eq:GQ_expression}
G^Q
=
\frac{\partial \mathcal{L}}{\partial W_Q}
=
\frac{1}{\sqrt{d}}\, X^{\top} G_M K ,
\end{equation}
and
\begin{equation}
\label{eq:GK_expression}
G^K
=
\frac{\partial \mathcal{L}}{\partial W_K}
=
\frac{1}{\sqrt{d}}\, X^{\top} G_M^{\top} Q .
\end{equation}

Next, consider the gradient with respect to the logits $Z$.
We have
\begin{equation}
\label{eq:GZ_norm}
\|G_Z\|_F
=
\sqrt{
\sum_{i=1}^n
\frac{1}{n^2}
\,
\|p_i - y_i\|_2^2
}
\le
\sqrt{
\sum_{i=1}^n
\frac{1}{n^2}}
\, 2(C-1)\exp\!\bigl(-\tau^V \omega_{\min}\bigr)
.
\end{equation}
Therefore,
\begin{equation}
\label{eq:GZ_bound}
\|G_Z\|_F
\le
\frac{\sqrt{2(C-1)}}{\sqrt{n}}
\exp\!\bigl(-\tau^V \omega_{\min}\bigr).
\end{equation}

Let
\begin{equation}
\label{eq:CH_def}
C_H
=
\frac{2(C-1)}{\sqrt{n}} \, \|W_C\|_F .
\end{equation}
Then it follows that
\begin{equation}
\label{eq:GH_bound}
\|G_H\|_F
\le
C_H \exp\!\bigl(-\tau^V \omega_{\min}\bigr).
\end{equation}

Moreover, we have
\begin{equation}
\label{eq:GV_norm_bound}
\|G^V\|_F
=
\|X^{\top} A^{\top} G_H\|_F
\le
\|X^{\top}\|_2 \, \|A^{\top}\|_2 \, \|G_H\|_F
\le
\sqrt{n}\,\lambda_{\max}(X)\,\|G_H\|_F .
\end{equation}

Define
\begin{equation}
\label{eq:CV_def}
C_V
=
2(C-1)\lambda_{\max}(X)\,\|W_C\|_F .
\end{equation}
Then it follows that
\begin{equation}
\label{eq:GV_exp_decay}
\|G^V\|_F
\le
C_V \exp\!\bigl(-\tau^V \omega_{\min}\bigr).
\end{equation}

Next, we compute
\begin{equation}
\label{eq:GQ_norm_bound}
\|G^Q\|_F
=
\frac{1}{\sqrt{d}}\,
\|X^{\top} G_M K\|_F
\le
\frac{1}{\sqrt{d}}\,
\lambda_{\max}^2(X)\,
\tau^K
\|G_M\|_F .
\end{equation}

Moreover, the Frobenius norm of $G_M$ satisfies
\begin{equation}
\label{eq:GM_norm_chain}
\begin{aligned}
\|G_M\|_F
&=
\sqrt{
\sum_{i=1}^n
\bigl\|
J(a_i)^{\top} G_{A,i}
\bigr\|_F^2
} \\
&\le
\sqrt{
\sum_{i=1}^n
\|J(a_i)\|_2^2 \, \|G_{A,i}\|_F^2
} \\
&\le
\max_{i} \|J(a_i)\|_2 \,
\sqrt{
\sum_{i=1}^n
\|G_{A,i}\|_F^2
} \\
&=
\max_{i} \|J(a_i)\|_2 \, \|G_A\|_F .
\end{aligned}
\end{equation}

Substituting the bound
$\|J(a_i)\|_2
\le
2(n-1)\exp\!\left(
-\frac{\gamma_{\min}}{\sqrt{d}}\,\tau^Q \tau^K
\right)$
into \eqref{eq:GM_norm_chain}, we obtain
\begin{equation}
\label{eq:GM_exp_bound}
\|G_M\|_F
\le
2(n-1)\exp\!\left(
-\frac{\gamma_{\min}}{\sqrt{d}}\,\tau^Q \tau^K
\right)
\|G_A\|_F .
\end{equation}

Moreover, the gradient with respect to the attention matrix satisfies
\begin{equation}
\label{eq:GA_norm_bound}
\|G_A\|_F
=
\|G_H V\|_F
\le
\|G_H\|_F \, \|V\|_F
\le
C_H \lambda_{\max}(X)\,\tau^V
\exp\!\bigl(-\tau^V \omega_{\min}\bigr) .
\end{equation}

Next, recall that
\begin{equation}
\label{eq:M_factorization}
M
=
\frac{QK^{\top}}{\sqrt{d}}
=
\frac{X W_Q W_K^{\top} X^{\top}}{\sqrt{d}}
=
\frac{\tau^Q \tau^K}{\sqrt{d}}
\, X \bar W_Q \bar W_K^{\top} X^{\top}
=
\frac{\tau^Q \tau^K}{\sqrt{d}} \, \bar M .
\end{equation}

For the $i$-th row of $\bar M$, define the margin
\begin{equation}
\label{eq:gamma_i_def}
\gamma_i
:=
\operatorname{gap}\!\bigl(\bar M_{i,:}\bigr),
\qquad
\gamma_{\min}
:=
\min_{i} \gamma_i .
\end{equation}

Consequently, for the original logits $M$, we have
\begin{equation}
\label{eq:gap_scaling}
\operatorname{gap}\!\bigl(M_{i,:}\bigr)
=
\frac{\tau^Q \tau^K}{\sqrt{d}} \, \gamma_i
\ge
\frac{\tau^Q \tau^K}{\sqrt{d}} \, \gamma_{\min} .
\end{equation}

From Lemma~\ref{lem:7}, we obtain the bound for the Jacobian matrix
\begin{equation}
\label{eq:J_bound_1}
\| J(a_i) \|_2 \le 2(1 - a_{\max}).
\end{equation}

From Lemma~\ref{lem:6}, we get the following bound for the Jacobian matrix:
\begin{equation}
\label{eq:J_bound_2}
\| J(a_i) \|_2
\le 2(1 - \alpha_{\max})
\le 2(n-1)\exp\!\left(-\operatorname{gap}(M_{i,:})\right).
\end{equation}

Thus, we can derive the following exponential decay bound:
\begin{equation}
\label{eq:J_exp_decay}
\| J(a_i) \|_2
\le 2(n-1)\exp\!\left(
- \frac{\gamma_{\min}}{\sqrt{d}} \, \tau^Q \tau^K
\right).
\end{equation}

Similarly, we can handle $\tau^V$ using the same approach.

Let
\begin{equation}
\label{eq:WV_def}
W_V = \tau^V \bar W_V.
\end{equation}
Then, the matrix $Z$ can be expressed as
\begin{equation}
\label{eq:Z_definition}
Z = A X W_V W_C
  = \tau^V A X \bar W_V W_C,
\end{equation}
and let
\begin{equation}
\label{eq:Z_bar_definition}
\bar Z = A X \bar W_V W_C.
\end{equation}

Analogously, we define the \emph{logit gap} on $\bar Z$ as
\begin{equation}
\label{eq:omega_i_def}
\omega_i
:= \bar Z_{i,y_i} - \max_{c \neq y_i} \bar Z_{i,c}.
\end{equation}

We define the minimum logit gap as
\begin{equation}
\label{eq:omega_min_def}
\omega_{\min} := \min_i \omega_i.
\end{equation}

From Lemma~\ref{lem:6}, we obtain the following bound on the probabilities:
\begin{equation}
\label{eq:prob_bound}
1 - p_{i,y_i}
\le (C-1)\exp\left(-\tau^V \omega_i\right)
\le (C-1)\exp\left(-\tau^V \omega_{\min}\right).
\end{equation}

Consider the Frobenius norm of the gradient with respect to $Z$:
\begin{equation}
\label{eq:GZ_norm_def}
\| G_Z \|_F^2
=
\left\| \frac{\partial \ell}{\partial Z} \right\|_F^2
=
\sum_{i=1}^n
\left\| \frac{\partial \ell}{\partial Z_i} \right\|_F^2
=
\sum_{i=1}^n
\frac{1}{n^2}
\,
\| p_i - y_i \|_2^2 .
\end{equation}

For each sample $i$, we have
\begin{equation}
\label{eq:pi_yi_bound}
\begin{aligned}
\| p_i - y_i \|_2^2
&=
(1 - p_{i,y_i})^2
+
\sum_{c \neq y_i} p_{i,c}^2 \\
&\le
(1 - p_{i,y_i})^2
+
\left( \sum_{c \neq y_i} p_{i,c} \right)^2 \\
&=
(1 - p_{i,y_i})^2
+
(1 - p_{i,y_i})^2 \\
&=
2(1 - p_{i,y_i})^2 .
\end{aligned}
\end{equation}
Therefore,
\begin{equation}
\label{eq:pi_yi_sqrt_bound}
\| p_i - y_i \|_2
\le
\sqrt{2}\,(1 - p_{i,y_i})
\le
2(C-1)\exp\!\bigl(-\tau^V \omega_{\min}\bigr).
\end{equation}

Next, combining the previously established bounds, we obtain
\begin{equation}
\label{eq:GQ_full_bound}
\begin{aligned}
\| G^Q \|_F
&\le
\frac{1}{\sqrt{d}}\,
\lambda_{\max}^2(X)\,
\tau^K
\cdot
2(n-1)
\exp\!\left(
- \frac{\gamma_{\min}}{\sqrt{d}} \, \tau^Q \tau^K
\right) \\
&\qquad\qquad\cdot
C_H \lambda_{\max}(X)\,
\tau^V
\exp\!\bigl(-\tau^V \omega_{\min}\bigr) \\
&\le
\frac{2(n-1)}{\sqrt{d}}\,
C_H \lambda_{\max}^3(X)\,
\exp\!\bigl(-\omega_{\min}\tau^V\bigr)
\exp\!\left(
- \frac{\gamma_{\min}}{\sqrt{d}} \, \tau^Q \tau^K
\right)
\cdot
\tau^K \tau^V .
\end{aligned}
\end{equation}

Similarly, we have
\begin{equation}
\label{eq:GK_full_bound}
\| G^K \|_F
\le
\frac{2(n-1)}{\sqrt{d}}\,
C_H \lambda_{\max}^3(X)\,
\exp\!\bigl(-\omega_{\min}\tau^V\bigr)
\exp\!\left(
- \frac{\gamma_{\min}}{\sqrt{d}} \, \tau^Q \tau^K
\right)
\cdot
\tau^Q \tau^V .
\end{equation}

Then the following upper bounds hold:
\begin{equation}
\label{eq:G_bounds_raw}
\begin{cases}
\| G^V \|_F 
\le
C_V \exp\!\bigl(-\omega_{\min}\tau^V\bigr), \\
\| G^K \|_F 
\le
\dfrac{2(n-1)}{\sqrt{d}}\,
C_H \lambda_{\max}^3(X)\,
\tau^Q \tau^V
\exp\!\bigl(-\omega_{\min}\tau^V\bigr)
\exp\!\left(
- \dfrac{\gamma_{\min}}{\sqrt{d}} \tau^Q \tau^K
\right), \\
\| G^Q \|_F 
\le
\dfrac{2(n-1)}{\sqrt{d}}\,
C_H \lambda_{\max}^3(X)\,
\tau^K \tau^V
\exp\!\bigl(-\omega_{\min}\tau^V\bigr)
\exp\!\left(
- \dfrac{\gamma_{\min}}{\sqrt{d}} \tau^Q \tau^K
\right).
\end{cases}
\end{equation}

The constants are defined as
\begin{equation}
\label{eq:constants_def}
C_V
=
2(C-1)\lambda_{\max}(X)\,\|W_C\|_F,
\qquad
C_H
=
\dfrac{2(C-1)}{\sqrt{n}}\,\|W_C\|_F .
\end{equation}

Let
\begin{equation}
\label{eq:CQK_def}
C_{QK}
:=
\dfrac{2(n-1)}{\sqrt{d}}\,
C_H \lambda_{\max}^3(X),
\qquad
\omega := \omega_{\min},
\qquad
\gamma := \dfrac{\gamma_{\min}}{\sqrt{d}} .
\end{equation}

Then \eqref{eq:G_bounds_raw} can be equivalently written as
\begin{equation}
\label{eq:G_bounds_simplified}
\begin{cases}
\| G^V \|_F
\le
C_V \exp(-\omega \tau^V), \\
\| G^K \|_F
\le
C_{QK}\, \tau^Q \tau^V
\exp(-\omega \tau^V)
\exp(-\gamma \tau^Q \tau^K), \\
\| G^Q \|_F
\le
C_{QK}\, \tau^K \tau^V
\exp(-\omega \tau^V)
\exp(-\gamma \tau^Q \tau^K).
\end{cases}
\end{equation}

Therefore, the squared Frobenius norm of the full gradient satisfies
\begin{equation}
\label{eq:full_grad_bound}
\begin{aligned}
\| \nabla_\theta \mathcal{L}(\theta) \|_F^2
&\le
\| G^V \|_F^2 + \| G^K \|_F^2 + \| G^Q \|_F^2 \\
&\le
C_V^2 \exp(-2\omega \tau^V)
+
C_{QK}^2
\Big[
(\tau^Q \tau^V)^2 + (\tau^K \tau^V)^2
\Big]
\exp(-2\omega \tau^V)
\exp(-2\gamma \tau^Q \tau^K).
\end{aligned}
\end{equation}

Since exponential decay dominates polynomial decay, there exists $p>2$ such that
\begin{equation}
\label{eq:exp_to_poly}
\exp(-2\omega \tau^V)
\le
(2\omega \tau^V)^{-p}
\le
(2\omega)^{-p} (\tau^V)^{-p},
\end{equation}
and
\begin{equation}
\label{eq:exp_to_poly_QK}
\exp(-2\gamma \tau^Q \tau^K)
\le
(2\gamma \tau^Q \tau^K)^{-p}
=
(2\gamma)^{-p} (\tau^Q \tau^K)^{-p}.
\end{equation}

From Theorem~\ref{thm:ssd_stability}, we have $\varepsilon = O(\eta/\tau)$.
Consequently, the one-step decrease of the loss satisfies
\begin{equation}
\label{eq:loss_decrease}
\begin{aligned}
\Delta \mathcal{L}(t)
&=
\mathcal{L}(t) - \mathcal{L}(t+1) \\
&\le
\eta \left(1 + \frac{2\beta}{2}\right)
\Big[
C_V^2 \exp(-2\omega \tau^V)
+
C_{QK}^2
\big( (\tau^Q)^2 + (\tau^K)^2 \big)
(\tau^V)^2
\exp(-2\omega \tau^V)
\exp(-2\gamma \tau^Q \tau^K)
\Big] \\
&\le
\eta \left(1 + \frac{2\beta}{2}\right)
\Bigg[
\dfrac{C_V^2}{(2\omega)^p}
\left(\dfrac{1}{\tau^V}\right)^p
+
\dfrac{2C_{QK}^2}{(4\omega\gamma)^p}
\left(\dfrac{1}{\tau^V}\right)^{p-2}
\left(\dfrac{1}{\tau^K}\right)^{2p-2}
\Bigg] \\
&\le
O(\varepsilon^\alpha),
\end{aligned}
\end{equation}
where
\begin{equation}
\label{eq:alpha_def}
\alpha = \min\{\, p,\; 3p - 4 \,\} = p .
\end{equation}

\end{proof}

\section{Experimental Setting}

\label{app:exp_details}

\subsection{GPT-2 on FineWeb}
\label{app:gpt2_variants_details}

\paragraph{Models and Architecture.}
We train GPT-2 Small (124M) and Medium (355M) models. While rooted in the standard decoder-only Transformer architecture \citep{radford2019language, vaswani2017attention}, we incorporate modern architectural enhancements to improve training stability and performance. Specifically, we follow the \texttt{modded-nanogpt} benchmark\footnote{\url{https://github.com/KellerJordan/modded-nanogpt}}, which replaces LayerNorm with \textbf{RMSNorm} \citep{zhang2019root}, adopts \textbf{Rotary Positional Embeddings (RoPE)} \citep{su2023roformerenhancedtransformerrotary}, and utilizes \textbf{Squared ReLU} activations.

The architectural configurations for the two scales are as follows:
\begin{itemize}[leftmargin=*]
    \item \textbf{Small (124M):} A modified configuration with $n_{\text{layer}}=12$, $d_{\text{model}}=768$, and $n_{\text{head}}=6$. Notably, this results in a head dimension of $d_{\text{head}}=128$, larger than the standard 64.
    \item \textbf{Medium (355M):} A standard configuration with $n_{\text{layer}}=24$, $d_{\text{model}}=1024$, and $n_{\text{head}}=16$, keeping the standard head dimension $d_{\text{head}}=64$.
\end{itemize}

\paragraph{Dataset and Tokenization.}
Training is performed on the FineWeb dataset \citep{penedo2024fineweb}, strictly adhering to the 10B token subset prescribed by the NanoGPT benchmark. This subset provides a high-quality, representative sample suitable for efficiency comparisons. We utilize the canonical GPT-2 Byte-Pair Encoding (BPE) tokenizer, processing input sequences with a context window of $L=1024$ tokens. To ensure consistent benchmarking, all evaluation metrics are reported on the standard 10M-token FineWeb validation partition.

\paragraph{Hyperparameters.}
To optimize computational efficiency, we follow the hyperparameter setting in the NanoGPT speedrun benchmark, which adopts specific settings for learning rates, weight decay, and scheduler phases for each model. We employ a split optimization strategy, applying distinct learning rates to the embeddings/head ($LR_{\text{head}}$) versus the transformer body ($LR_{\text{body}}$). The detailed configurations are summarized in Table~\ref{tab:gpt2_hyperparams}.

\begin{table}[h]
\centering
\caption{Hyperparameters for Modded GPT-2 experiments. \textbf{WSD} denotes the Warmup-Stable-Decay schedule (linear warmup, constant hold, linear decay). \textbf{Step Decay} drops the learning rate by a factor of 10 at specified milestones.}
\label{tab:gpt2_hyperparams}
\resizebox{\textwidth}{!}{%
\begin{tabular}{lccccccccc} 
\toprule
\textbf{Model} & \textbf{Steps} & \textbf{Tokens} & \textbf{Optimizer} & $LR_{\text{head}} / LR_{\text{body}}$ & \textbf{Body WD} & \textbf{LR Schedule} & \textbf{Warmup} \\
\midrule
Small & 20,400 & $\sim$10.7B & AdamW & $3.6\text{e-}3$ / $1.8\text{e-}3$ & 0.0 & {Step Decay (at 50\%, 75\%)} & 250 \\
Small & 5,100  & $\sim$2.6B  & AdamW & $3.6\text{e-}3$ / $1.8\text{e-}3$ & 0.0 & {WSD (Linear Decay)} & 250 \\
Small & 5,100  & $\sim$2.6B  & Muon  & $3.6\text{e-}3$ / $2.0\text{e-}3$ & 0.0 & {WSD (Linear Decay)} & 250 \\
Medium& 5,100  & $\sim$2.6B  & AdamW & $3.0\text{e-}3$ / $1.5\text{e-}3$ & 0.125 & {WSD (Linear Decay)} & 500 \\
\bottomrule
\end{tabular}%
}
\end{table}

\paragraph{Optimization Details.}
\begin{itemize}[leftmargin=*, topsep=2pt, itemsep=1pt]
    \item \textbf{Small Long-Run:} Trained for $\sim$10.7B tokens using a \textbf{Step Decay} schedule. To maximize convergence on the larger data budget, the learning rate is dropped by a factor of 10 at 50\% and 75\% of the total training steps.
    \item \textbf{Small Short-Runs:} We conduct a controlled comparison between \textbf{AdamW} and the \textbf{Muon} optimizer over a $\sim$2.6B token budget. Both runs utilize a \textbf{WSD} (Warmup-Stable-Decay) schedule with a linear warmdown. Weight decay is explicitly disabled ($\lambda=0$) for both optimizers to strictly isolate the algorithmic differences.
    \item \textbf{Medium Experiment:} Trained for $\sim$2.6B tokens using AdamW with a \textbf{WSD} schedule. Unlike the Small experiments, we apply specific regularization to the larger model by introducing weight decay ($\lambda=0.125$) to the Transformer Body parameters, while keeping $\lambda=0$ for the embeddings and head.
\end{itemize}

\subsection{Pre-training Setup: LLaMA on C4}
\label{app:llama_details}

\paragraph{Model Architecture.}
We implement the LLaMA architecture \citep{touvron2023llama}, which features RMSNorm for pre-normalization, SwiGLU activations, and Rotary Positional Embeddings (RoPE) \citep{su2023roformerenhancedtransformerrotary}. The implementation is based on the HuggingFace Transformers library \citep{wolf2020transformers}. In this study, we experiment with two distinct model scales: 0.5B and 2B parameters. The detailed hyperparameters and architectural specifications are summarized in Table~\ref{tab:llama_arch}.

\begin{table}[h]
\centering
\caption{Detailed architectural configurations and hyperparameters for the LLaMA models. $LR_{\max}$ indicates the maximum learning rate applied during the cosine schedule.}
\label{tab:llama_arch}
\begin{tabular}{lcccccc}
\toprule
\textbf{Model} & \textbf{Params} & $n_{\text{layer}}$ & $d_{\text{model}}$ & $n_{\text{head}}$ & $d_{\text{ff}}$ & $LR_{\max}$ \\
\midrule
LLaMA-0.5B & 0.5B & 15 & 1280 & 20 & 5120 & $6.0 \times 10^{-4}$ \\
LLaMA-2B   & 2B   & 30 & 2048 & 32 & 8192 & $3.0 \times 10^{-4}$ \\
\bottomrule
\end{tabular}
\end{table}

\paragraph{Dataset and Tokenization.}
Our models are pre-trained on the C4 (Colossal Clean Crawled Corpus) dataset \citep{raffel2020exploring}. For text processing, we employ a SentencePiece tokenizer configured with a vocabulary size of $V=32,100$. The input sequence length is fixed at $L=2048$ tokens. To assess model performance, we report both perplexity and loss metrics on the official C4 validation split.

\paragraph{Optimization and Training.}
Our training configuration aligns with \citet{zhao2024galore}. We utilize the AdamW optimizer with hyperparameters set to $\beta_1=0.9$, $\beta_2=0.95$, and $\epsilon=10^{-8}$. Gradient clipping is applied with a threshold of $1.0$.

Regarding regularization, we employ two distinct settings:
\begin{itemize}[leftmargin=*, topsep=2pt, itemsep=1pt]
    \item \textbf{Standard Setting:} For the LLaMA-2B and the primary LLaMA-0.5B baseline, we apply a standard weight decay of $\lambda=0.1$.
    \item \textbf{No-WD Variant:} We train an additional LLaMA-0.5B variant with weight decay disabled ($\lambda=0$) while keeping all other hyperparameters identical, serving as a non-regularized baseline for comparison.
\end{itemize}

The learning rate is managed via a cosine decay schedule, which consists of:
\begin{enumerate}[leftmargin=*, topsep=2pt, itemsep=1pt]
    \item A linear warmup phase for the first 1,000 steps;
    \item A peak learning rate ($LR_{\max}$) as specified in Table~\ref{tab:llama_arch};
    \item A cosine decay phase reducing the rate to a minimum of $lr_{\min} = 0.1 \times LR_{\max}$.
\end{enumerate}
Training is conducted with a global batch size of 512 sequences for a total duration of 100,000 steps, processing approximately 105 billion tokens.

\paragraph{System Implementation.}
We leverage PyTorch Fully Sharded Data Parallel (FSDP) for distributed training. Specifically, we adopt the \texttt{HYBRID\_SHARD} strategy combined with \texttt{bfloat16} mixed precision. In this setup, model parameters and buffers are stored in FP32, while gradients and communication operations are performed in \texttt{bfloat16}. CPU offloading is disabled to maximize training efficiency.



\end{document}